\documentclass[a4paper,12pt,reqno]{amsart}
\usepackage{amsmath,amssymb,amsfonts}
\usepackage{amsthm}
\usepackage{graphicx}
\usepackage{url}
\usepackage{hyperref}
\usepackage{geometry}
\usepackage{algorithm}
\usepackage{algpseudocode}
\usepackage{xcolor}
\usepackage{threeparttable}
\usepackage{enumerate}
\usepackage{dsfont}
\usepackage[export]{adjustbox}
\usepackage{eucal}
\usepackage[utf8]{inputenc}
\usepackage{tabularx}
\usepackage{mathtools}
\usepackage{comment}
\usepackage{booktabs}
\usepackage[numbers,sort&compress]{natbib}
\usepackage{array}
\usepackage{enumitem}
\usepackage{bm}
\usepackage{float}
\usepackage{multirow}
\usepackage{rotating}
\usepackage{subcaption}

\newtheorem{theorem}{Theorem}[section]
\newtheorem{proposition}[theorem]{Proposition}
\newtheorem{corollary}[theorem]{Corollary}
\newtheorem{lemma}[theorem]{Lemma}
\newtheorem{definition}[theorem]{Definition}
\newtheorem{remark}[theorem]{Remark}

\title[Fractional Heat Kernel for Semi-Supervised Graph Learning]{Fractional Heat Kernel for Semi-Supervised Graph Learning with Small Training Sample Size}

\author[F. Bozorgnia et al.]{Farid Bozorgnia, Vyacheslav Kungurtsev, Shirali Kadyrov, Mohsen Yousefnezhad}
\address{Department of Mathematics, New Uzbekistan University}
\email{f.bozorgnia@newuu.uz}
\address{Department of Computer Science, Faculty of Electrical Engineering, Czech Technical University in Prague, Czech Republic}
\email{kunguvya@fel.cvut.cz}

\address{Department of Mathematics, SDU University, Kaskelen, Kazakhstan}
\email{shirali.kadyrov@sdu.edu.kz}

\address{Department of Mathematics, Shiraz University}
\email{mohsenyousefnezhad@gmail.com}
\date{}

\begin{document}

\begin{abstract}
We develop a source-driven fractional heat-kernel framework for
semi-super\-vised graph learning that combines nonlocal propagation with
sustained label information. A fixed nonzero label source compatible with
the Laplacian null space prevents asymptotic collapse into that space,
providing a mechanism for mitigating oversmoothing at long diffusion times.
The fractional order controls the relative modal attenuation and the
spectral weighting of the sustained response, while the diffusion time sets
the propagation horizon. We characterize conservation laws and equilibria
on normalized and disconnected graphs, develop a null-space deflation, and
analyze the approximation of the propagators. On Two-Moon, fractional
orders improve source-free propagation, while compatible source-driven
diffusion exceeds $96\%$ mean accuracy with one training label per class at
orders $0.8$ and $1$. On Cora and CiteSeer, the source-driven pipeline
improves mean accuracy over GAT by $9.2$ and $8.0$ percentage points at one
training label per class, with model selection on $500$ labeled validation
nodes and closely comparable classical and fractional pipeline
configurations; it matches GAT on PubMed and trails it at ten and twenty
labels per class on Cora. Within GraphHeat, validation-based exponent
selection at a common diffusion time yields a paired gain of $1.45$
percentage points at one label per class.
\end{abstract}

\maketitle

\section{Introduction}
Graphs provide a natural representation for structured data in applications ranging from social networks to molecules in drug discovery~\cite{hamilton2017representation, bronstein2017geometric}. Learning from such data is challenging, particularly in semi-supervised scenarios where only a small fraction of nodes are labeled, as seen in large social networks or protein function prediction tasks~\cite{zhu2003semi, zhou2004learning}. Graph Neural Networks (GNNs) have made progress on these problems~\cite{scarselli2008graph, wu2020comprehensive}, but they often suffer from over-smoothing, where node features become indistinguishable after multiple iterations of message passing~\cite{li2018deeper, oono2020graph}. This makes deep GNNs difficult to tune, especially for tasks that require information to propagate across distant nodes.

Heat-kernel propagation uses the graph geometry through the matrix exponential $e^{-tL}$~\cite{grigoryan2009heat}. Its diffusion time sets the propagation scale, and its spectrum determines how strongly different signal components are smoothed~\cite{chung1997spectral,kondor2002diffusion}. Our objective is to propagate information over the graph while sustaining label-dependent variation at long diffusion times. We achieve this through a fixed compatible label source: for $F\ne0$ and $\Pi_0F=0$, the non-null component of the state converges to $(L^s)^\dagger F\ne0$ (Theorem~\ref{thm:bounded}), which prevents asymptotic collapse into the Laplacian null space. The fractional order controls the relative attenuation of the transient modes and the spectral weighting of this sustained response, providing a tunable balance between low-frequency smoothing and retention of higher-frequency variation. The construction builds on the use of partial differential equations to motivate graph-learning methods~\cite{BertozziDrenskaLatzThorpe2025, BBE, BFAE}.

Our method builds upon a recent line of work. Personalized propagation of neural predictions decouples prediction from propagation through a personalized PageRank operator~\cite{klicpera2019predict}, and graph diffusion convolution generalizes the propagation operator to a broader family of diffusion kernels~\cite{klicpera2019diffusion}. GraphHeat adopts heat-kernel filtering to realize spectral diffusion over the graph~\cite{xu2020graph}, while GRAND formulates representation propagation as a continuous-depth process governed by an ordinary differential equation involving the graph Laplacian~\cite{chamberlain2021grand}. In contrast to classical diffusion mechanisms, our framework introduces fractional heat-kernel diffusion with a source-driven propagation term, enabling a tunable nonlocal smoothing effect under sparse supervision. Closest to our propagation model is GRAND++~\cite{thorpe2022grandpp}, which adds a source term to the graph diffusion ODE for learning at low labeling rates, while GraphHeat is the closest on the operator axis, since it uses the heat kernel as a low-pass spectral filter for semi-supervised node classification.

Poisson learning~\cite{calder2020poisson} provides a direct antecedent for source-based learning at very low label rates: centered label sources and sinks define a solvable graph Poisson equation. Our compatible-source dynamics have the limit $\Pi_0U_0+(L^s)^\dagger F$, whose range component solves the fractional Poisson system; at $s=1$ with the combinatorial Laplacian the source equation connects directly to Poisson learning. The present formulation incorporates fractional spectral weighting, finite diffusion times, Laplacian-dependent componentwise centering and the null-space content of the initial state.

The contribution is a fractional propagation framework that combines sustained label supervision with explicit control over spectral smoothing. We characterize its conserved label functionals and source-driven equilibria, including the componentwise compatibility needed on normalized and disconnected graphs. The analysis explains how a fixed nonzero compatible source prevents asymptotic null-space collapse and how the order reshapes the retained information across positive frequencies. We develop deflation and approximation tools and combine source-driven propagation with learned graph representations. The experiments examine the roles of the source, the order and graph refinement: fractional orders improve source-free diffusion on Two-Moon; the source-driven pipeline yields substantial low-label gains over GAT on Cora and CiteSeer; and validation-based exponent selection improves GraphHeat at the sparsest budget. Matched classical controls quantify the additional effect of the order within each setting. Sections~\ref{sec:gnn} and~\ref{sec:framework} detail the relationships to GraphHeat and GRAND++.

Two recent works apply fractional calculus to graph neural dynamics with the same stated motivation. A fractional graph Laplacian approach to oversmoothing places a singular-value power of a normalized graph operator inside a continuous-depth model, with the order learned jointly with the network weights and the channel mixing inside the diffusion, and evaluates across homophily levels~\cite{maskey2023flode}. Fractional-order graph neural dynamical networks replace the first-order time derivative by a Caputo derivative, giving time-fractional counterparts of GRAND and GRAND++~\cite{kang2024frond}; that construction is time-fractional, whereas the operator used here is space-fractional, and the spectral analysis of Section~\ref{subsec:quantitative_fractional_order} applies only to the latter. Relative to~\cite{maskey2023flode}, the present work takes the fractional power of a symmetric positive semidefinite Laplacian by functional calculus on its eigenbasis, requires the source to be compatible with the null space of that operator, which is the content of Theorem~\ref{thm:mass}, and  selects the order on validation data. A controlled comparison in Section~\ref{subsec:graphheat_controlled} isolates
the exponent within GraphHeat's architecture at a common diffusion time and a
matched sparsity budget.

The remainder of this paper is organized as follows. Section~\ref{sec:spectral} analyzes the spectral decomposition of the heat kernel matrix and develops efficient approximation techniques using Chebyshev polynomials. Section~\ref{sec:fractional} introduces the fractional heat kernel, quantifies how the fractional order reshapes the spectral filter, and states the convergence rate of the fractional Chebyshev approximation. Section~\ref{sec:gnn} integrates heat kernel diffusion into modern GNN architectures and states the training mechanics. Section~\ref{sec:framework} presents our framework for fractional Laplacian semi-supervised learning with continuous supervision, together with an optional self-training extension. Section~\ref{sec:experiments} evaluates the method on a synthetic benchmark, three citation networks and four heterophilic graphs. Section~\ref{sec:limitations} discusses limitations and failure modes, and Section~\ref{sec:conclusion} concludes.

\section{Spectral Decomposition Analysis of the Heat Kernel Matrix}\label{sec:spectral}

Diffusion kernels for learning on discrete input spaces were introduced by
Kondor and Lafferty~\cite{kondor2002diffusion}, who exponentiate a graph
Laplacian to obtain a positive semidefinite kernel; heat kernels on graphs had
been studied earlier within spectral graph theory~\cite{chung1997spectral}.
Coifman and Lafon~\cite{coifman2006diffusion} introduced diffusion maps, which
embed the vertices through the spectral decomposition of a Markov diffusion
operator. Heat-kernel filtering was placed inside a graph convolutional
architecture by Xu et al.~\cite{xu2020graph}; Berberidis et
al.~\cite{berberidis2019adaptive} take a different route, learning
class-specific combinations of random-walk landing probabilities in an adaptive
label-diffusion scheme rather than inside a neural architecture. Fractional
Laplacians model nonlocal interactions~\cite{EvangelistaLenzi2018}, and their
use inside graph neural architectures is recent~\cite{maskey2023flode,kang2024frond}.
This section analyzes the spectral properties of the heat kernel matrix
$e^{-t\mathbf{L}}$ and presents approximation techniques for computation on
large graphs.

We first fix notation. We consider an undirected, weighted graph
$G=(V,E,W)$ with $V=\{x_1,\dots,x_n\}$ the set of nodes and $E$ the edge set.
The weight matrix is $\mathbf W$, where $W_{ij}>0$ is the weight of the edge
between nodes $i$ and $j$ and $W_{ij}=0$ if no edge exists. The degree matrix
$\mathbf D$ is diagonal with $D_{ii}=\sum_jW_{ij}$. We use the combinatorial
graph Laplacian $\mathbf L=\mathbf D-\mathbf W$, whose heat kernel has unit row
sums and therefore conserves the ordinary arithmetic mass. The spectral
decomposition of $\mathbf L$ yields eigenvalues
$0=\lambda_1\le\lambda_2\le\cdots\le\lambda_n$ with orthonormal eigenvectors
$\phi_1,\dots,\phi_n$, arranged in $\mathbf U=[\phi_1,\dots,\phi_n]$. We write
$r=\dim\ker(\mathbf L)$, $\Pi_0$ for the orthogonal projector onto
$\ker(\mathbf L)$, and $\lambda_+$ for the smallest positive eigenvalue, which
equals $\lambda_2$ only when the graph is connected. For functions
$u:V\to\mathbb R$ we use the vector
$\mathbf u=[u(x_1),\dots,u(x_n)]^{\top}$, and
$\langle\phi_k,\mathbf u\rangle=\sum_{i=1}^n\phi_k(x_i)u(x_i)$. The all-ones
vector is $\mathbf 1$. The heat kernel matrix is
$\mathbf H_t=e^{-t\mathbf L}$ with entries $H_t(x_i,x_j)$ or $H_t(i,j)$. We
write $U^{0}$ for the initial condition and $U(t)$ for the evolved state. All
norms are $\ell_2$ norms unless stated otherwise, and
$\operatorname{vol}(G)=\sum_{i\in V}d_i$. Finally,
$\mathds{1}_l\in\mathbb R^{l\times1}$ is a column vector of ones and
$\mathbf 0_{(n-l)\times k}$ a zero matrix for the unlabeled nodes.

The spectral theory of Sections~\ref{sec:spectral} and~\ref{sec:fractional}
holds for any symmetric positive semidefinite graph Laplacian, since it uses
only the eigendecomposition and the resulting filter
$g_s(\lambda)=e^{-t\lambda^s}$. We use the combinatorial Laplacian for
intuition and for the conservation statements, because its null space is
spanned by the indicator vectors of the connected components, which makes the
mass interpretation transparent. The experiments in
Section~\ref{sec:experiments} use the symmetric normalized Laplacian
$L_{\mathrm{sym}}=I-D^{-1/2}WD^{-1/2}$, for three reasons: its spectrum lies in
$[0,2]$, so the Chebyshev mapping and the truncation bounds are uniform across
graphs; it is the standard operator for spectral GCN
filtering~\cite{KipfWelling2017GCN}; and on the graphs used here its spectrum
extends to both sides of the fixed point $\lambda=1$ of
Proposition~\ref{prop:rotation}, which is where the fractional order reshapes
the filter. The inclusion $\operatorname{spec}(L_{\mathrm{sym}})\subseteq[0,2]$ only
permits eigenvalues above one; whether they occur depends on the graph, and
Table~\ref{tab:label_spectral} reports the measured $\lambda_{\max}$ of each
experimental graph. The price of the normalization is that the null space is
spanned by $D^{1/2}\mathbf 1$ componentwise rather than by $\mathbf 1$, so
the conserved label functional changes accordingly; this is made precise in
Theorem~\ref{thm:mass} and Remark~\ref{rem:masspractice}. Wherever a specific
Laplacian is assumed, we state it.

\subsection{Spectral Properties of the Heat Kernel}
Let $L$ be a symmetric graph Laplacian with eigenpairs $(\lambda_k,\phi_k)$,
$k=1,\dots,n$, where $0=\lambda_1\le\cdots\le\lambda_n$ and the $\phi_k$ form
an orthonormal eigenbasis of $\mathbb R^n$~\cite{chung1997spectral}. Let
$U=[\phi_1,\dots,\phi_n]$ and let $e^{-t\Lambda}$ be diagonal with entries
$e^{-t\lambda_k}$. By the spectral theorem,
\begin{equation}
e^{-tL}=Ue^{-t\Lambda}U^{-1}=\sum_{k=1}^n e^{-t\lambda_k}\phi_k\phi_k^{\top},
\end{equation}
using $U^{-1}=U^{\top}$. Thus $e^{-tL}$ multiplies each eigenmode $\phi_k$ by
$e^{-t\lambda_k}$, so modes with large eigenvalue decay strongly while modes
with small eigenvalue are comparatively preserved. Diffusion kernels on graphs
can be viewed as a discretization of Gaussian kernels in Euclidean
space~\cite{xu2020graph}, and $t$ controls the diffusion scale, interpolating
between local and global geometry.

The kernel induces a diffusion distance on the vertices. For the combinatorial
Laplacian on a connected graph, $\phi_1$ is constant, and orthonormality of the
eigenvectors gives
\begin{equation}\label{eq:diffdist}
D_t(i,j)^2
:=\sum_{\ell\in V}\bigl(H_t(i,\ell)-H_t(j,\ell)\bigr)^2
=\sum_{k=2}^{n}e^{-2t\lambda_k}\bigl(\phi_k(i)-\phi_k(j)\bigr)^2 .
\end{equation}
Following the Diffusion Maps construction~\cite{coifman2006diffusion}, one
embeds node $x_i$ by retaining the leading coordinates,
\begin{equation}\label{eq:embed}
\Psi^{(m)}_t(x_i)=\left(e^{-t\lambda_2}\phi_2(x_i),\,
e^{-t\lambda_3}\phi_3(x_i),\dots,e^{-t\lambda_m}\phi_m(x_i)\right).
\end{equation}
The embedding approximates the diffusion distance rather than preserving it:
by \eqref{eq:diffdist},
\begin{align*}
D_t(i,j)^2-\bigl\|\Psi^{(m)}_t(x_i)-\Psi^{(m)}_t(x_j)\bigr\|_2^2
&=\sum_{k=m+1}^{n}e^{-2t\lambda_k}\bigl(\phi_k(i)-\phi_k(j)\bigr)^2\\
&\le e^{-2t\lambda_{m+1}}\sum_{k=m+1}^{n}\bigl(\phi_k(i)-\phi_k(j)\bigr)^2 ,
\end{align*}
which is nonnegative and in general nonzero. For the symmetric normalized
Laplacian the null vector is not constant, so the first coordinate cannot be
discarded without changing \eqref{eq:diffdist}, and the corresponding
construction uses the degree-weighted distance of the Markov diffusion
map~\cite{coifman2006diffusion}.  

Writing the action of the kernel on a vector $u$ with components $u(x_j)$, the
$i$-th component of $e^{-tL}u$ is
\begin{equation}\label{e2}
(e^{-tL}u)(x_i)=\sum_{k=1}^n e^{-t\lambda_k}\langle\phi_k,u\rangle\phi_k(x_i),
\end{equation}
where $\langle\phi_k,u\rangle=\sum_{j=1}^n\phi_k(x_j)u(x_j)$. In matrix form
$u(t)=Ue^{-t\Lambda}U^{\top}u(0)$. Writing
$H_t(x_i,x_j):=(e^{-tL})_{ij}$, \eqref{e2} reads
\[
(e^{-tL}u)(x_i)=\sum_{j\in V}H_t(x_i,x_j)u(x_j),
\qquad
H_t(x_i,x_j)=\sum_k e^{-t\lambda_k}\phi_k(x_i)\phi_k(x_j).
\]
Here $H_t$ is a fundamental solution: for $u_0=\delta_i$, a unit mass at
$x_i$, the quantity $H_t(x,x_i)=(e^{-tL}\delta_i)(x)$ is the solution at $x$ at
time $t$. The matrix $H_t$ is symmetric with nonnegative entries, and for each
fixed $i$,
\begin{equation}\label{eq:conservation}
\sum_j H_t(x_i,x_j)=1,
\end{equation}
the conservation of total mass under heat flow~\cite{grigoryan2009heat}.

The diffusion affinity $H_t(x_i,x_j)$ accumulates the contributions of all
walks from $i$ to $j$, weighted by their transition rates: for small $t$ it
is appreciable only for $j=i$ and the neighbours of $i$, and on a weighted
graph it is not a monotone function of graph distance. Its multihop
structure is expressed by the uniformization identity
\begin{equation}\label{eq:uniformization}
e^{-tL}=e^{-\alpha t}\sum_{k=0}^{\infty}\frac{(\alpha t)^k}{k!}
\left(I-\frac{L}{\alpha}\right)^{k},
\qquad \alpha\ge\max_i d_i,\ \alpha>0,
\end{equation}
a weighted sum over transition counts~\cite{lovasz1993random,burioni2005random}.
This identity makes heat-kernel propagation a global weighted average:
distant vertices contribute through longer walks, with weights determined
by the transition structure of the graph.

The operator $e^{-tL}$ is a self-adjoint contraction, and its application
improves regularity~\cite{grigoryan2009heat}. For every $u$,
\[
\|L^{1/2}e^{-tL}u\|_2^2=\sum_k\lambda_k e^{-2t\lambda_k}|\langle u,\phi_k\rangle|^2
\le e^{-2t\lambda_+}\|L^{1/2}u\|_2^2\le\|L^{1/2}u\|_2^2 .
\]
The Dirichlet energy is therefore nonincreasing along the flow, and is
substantially reduced only when $t\lambda_+$ is not small: for $u=\phi_k$ the
ratio is exactly $e^{-2t\lambda_k}$, which tends to one as $t\downarrow0$. This
is the $H^1$ regularization effect of heat
diffusion~\cite{belkin2006manifold}. Equivalently, heat diffusion is the
gradient flow of one half of the Dirichlet energy, since
$\nabla(u^{\top}Lu)=2Lu$. A related one-step implicit Euler, or Tikhonov,
regularization is
\[
u^\star=\arg\min_u\left\{\tfrac12\|u-u_0\|_2^2+\tfrac{t}{2}u^{\top}Lu\right\},
\qquad u^\star=(I+tL)^{-1}u_0 ,
\]
so both $e^{-tL}u_0$ and $(I+tL)^{-1}u_0$ suppress high-frequency modes, through
different spectral filters.

The matrix exponential $e^{-tL}$ is the fundamental solution of the graph heat
equation: if $u_0:V\to\mathbb R$, then $u(t)=e^{-tL}u_0$ is the unique solution
of $du/dt+Lu=0$. For small $t$, $e^{-tL}-I+tL=O(t^2)$, so $u$ initially changes
at rate $-Lu_0$, that is, each node's value moves toward the graph-weighted
average of its neighbors. As $t$ grows, every mode with positive eigenvalue
decays and
\[
e^{-tL}u\longrightarrow\Pi_0u
\qquad(t\to\infty),
\]
the orthogonal projection onto $\ker(L)$. For the combinatorial Laplacian this
is the componentwise average, $(\Pi_0u)_i=|C_a|^{-1}\sum_{j\in C_a}u_j$ for
$i\in C_a$, where $C_1,\dots,C_r$ are the connected components. On a connected
graph $r=1$ and the limit is the constant function equal to the average of the
initial values~\cite{lovasz1993random}: if $0=\lambda_1<\lambda_2$, then
\begin{equation}
\|e^{-tL}u-\bar u\,\mathbf 1\|\le e^{-t\lambda_2}\|u-\bar u\,\mathbf 1\|,
\end{equation}
with $\bar u=\tfrac1n\mathbf 1^{\top}u$; on a disconnected graph the same
estimate holds componentwise with $\lambda_+$ in place of $\lambda_2$. For
finite $t$ the operator preserves the low-frequency structure while
attenuating the high frequencies exponentially, so local features are retained
for small $t$ and global mixing occurs for large $t$~\cite{avramidi2015heat},
which permits tuning $t$ against the tradeoff between overtraining and
oversmoothing.

The heat kernel is also the transition density of a continuous-time random walk,
and the generator fixes both the waiting rates and the stationary law. The two
normalizations differ and should not be combined. The kernel $e^{-tL}$ with
$L=D-W$ is generated by $Q_{\mathrm c}=W-D$: the walker at $x_i$ waits an
exponentially distributed time of rate $d_i$, then jumps to a neighbor with
probability $W_{ij}/d_i$, and on a connected graph the stationary law is
uniform. The random-walk normalization corresponds instead to the generator
$Q_{\mathrm{rw}}=D^{-1}W-I$, whose negative is the random-walk Laplacian: the
waiting rate is one at every vertex and the stationary law is
$\pi_j=d_j/\operatorname{vol}(G)$. The degree-weighted stationary law therefore
belongs to the second process, not to the kernel $e^{-tL}$ used
here~\cite{lovasz1993random,burioni2005random}. For either generator,
\begin{equation}
P(X_t=j\mid X_0=i)=(e^{tQ})_{ij},
\end{equation}
and the semigroup property gives the Chapman-Kolmogorov identity
$H_{t+s}=H_tH_s$. This probabilistic reading explains why $H_t(x_i,x_j)$
accumulates multi-hop connectivity: longer walks between $x_i$ and $x_j$
contribute to the total probability through the superposition in
\eqref{eq:uniformization}.

\subsection{Approximation Techniques}
On large graphs the full eigendecomposition of $L$ costs $O(n^3)$ operations
and $O(n^2)$ storage. We present two approximations.

\subsubsection{Truncated Spectral Decomposition}
Let $P_m=\sum_{k=1}^m\phi_k\phi_k^{\top}$ with $1\le m<n$, and approximate
\begin{equation*}
e^{-tL}u\approx\sum_{k=1}^m e^{-t\lambda_k}\langle u,\phi_k\rangle\phi_k
=e^{-tL}P_mu .
\end{equation*}
The error is the propagated tail, and
\[
\bigl\|e^{-tL}(I-P_m)u\bigr\|_2^2
=\sum_{k=m+1}^{n}e^{-2t\lambda_k}|\langle u,\phi_k\rangle|^2
\le e^{-2t\lambda_{m+1}}\bigl\|(I-P_m)u\bigr\|_2^2
\le e^{-2t\lambda_{m+1}}\|u\|_2^2 ,
\]
by Parseval's identity and orthonormality. The error is controlled by the
omitted signal energy and the cutoff eigenvalue $\lambda_{m+1}$, and it
vanishes when $u\in\operatorname{range}(P_m)$. The exact worst-case error
over unit signals is
\[
\bigl\|e^{-tL}(I-P_m)\bigr\|_2=e^{-t\lambda_{m+1}} ,
\]
so the approximation becomes more accurate as the diffusion time or the
cutoff eigenvalue increases, and spectral separation can make it effective
at low rank on graphs with community structure~\cite{von2007tutorial}.

\subsubsection{Chebyshev Polynomial Approximation}
Forming $e^{-tL}$ directly costs $O(n^3)$ operations and $O(n^2)$ storage for
the dense $H_t$~\cite{moler2003nineteen}. The matrix exponential can instead be
approximated by Chebyshev polynomials, which avoid the eigendecomposition while
maintaining accuracy for moderate $t$~\cite{mason2002chebyshev,trefethen2000spectral}.
The Chebyshev polynomials of the first kind satisfy
\begin{equation*}
T_{k+1}(x)=2xT_k(x)-T_{k-1}(x),\qquad T_0(x)=1,\ T_1(x)=x,
\end{equation*}
and are orthogonal on $[-1,1]$ with weight $(1-x^2)^{-1/2}$. Mapping
$[0,\lambda_{\max}]$ to $[-1,1]$ by
$T_k^*(x)=T_k\!\left(\tfrac{2x}{\lambda_{\max}}-1\right)$, we approximate
$e^{-tx}$ by $p_m(x)=\sum_{k=0}^m c_kT_k^*(x)$ with
coefficients~\cite{trefethen2000spectral}
\begin{align*}
c_0 &= \frac{1}{\pi}\int_{-1}^1 e^{-t\frac{\lambda_{\max}(x+1)}{2}}\frac{dx}{\sqrt{1-x^2}}, &
c_k &= \frac{2}{\pi}\int_{-1}^1 e^{-t\frac{\lambda_{\max}(x+1)}{2}}T_k(x)\frac{dx}{\sqrt{1-x^2}},\quad k\ge1,
\end{align*}
computed to high accuracy by Gauss-Chebyshev quadrature~\cite{higham2008functions}.

In an implementation the polynomial is applied to a signal rather than
assembled as a matrix. With $\widetilde L=\tfrac{2L}{\lambda_{\max}}-I$, set
\begin{equation}\label{eq:chebvec}
v_0=u,\qquad v_1=\widetilde Lu,\qquad
v_{k+1}=2\widetilde Lv_k-v_{k-1},\qquad
p_m(L)u=\sum_{k=0}^m c_kv_k ,
\end{equation}
which is the filtering recurrence of Defferrard et
al.~\cite{defferrard2016convolutional}. For sparse $L$ this costs
$O\bigl(m(\operatorname{nnz}(L)+n)\bigr)$ operations per signal and $O(n)$
auxiliary storage when the sum is accumulated inside the recurrence, and the
cost scales linearly in the number of feature columns. The corresponding matrix
recurrence $T^*_{k+1}(L)=2\widetilde LT^*_k(L)-T^*_{k-1}(L)$ is the same
identity but is not the same computation: it produces the polynomial in
$L$ explicitly, with fill-in, and storing the resulting dense kernel still
costs $O(n^2)$. This distinction matters for the sparsification of
Section~\ref{subsec:implementation_details}, where the entries of the kernel
themselves are required rather than its action.

For $f(x)=e^{-tx}$, which is entire, the truncation error converges
geometrically~\cite{trefethen2000spectral}:
\begin{equation*}
\|e^{-tL}u-p_m(L)u\|_2\le C\rho^{-m}\|u\|_2,
\end{equation*}
where $\rho>1$ is set by the Bernstein ellipse. The fractional case is analyzed
separately in Section~\ref{sec:fraccomp}, where this geometric rate is shown to
fail.

\section{Fractional Heat Kernel}\label{sec:fractional}
The fractional Laplacian $L^{s}$, $s>0$, provides a tunable spectral
response and, for non-integer orders, nonlocal spatial coupling. For a fixed
symmetric positive semidefinite $L$ the propagators form the semigroup
$e^{-(t+r)L^{s}}=e^{-tL^{s}}e^{-rL^{s}}$~\cite{higham2008functions}. The
evolution is first order in time; its fractional structure lies in the
spatial graph operator. Time-fractional graph dynamics, such
as~\cite{kang2024frond}, use a fractional time derivative instead.

\subsection{Fractional Laplacian and Heat Kernel Definitions}
The operator $L^{s}$ is self-adjoint and positive semidefinite on $\ell^2(V)$:
\[
\langle u, L^{s}v\rangle=\langle L^{s}u,v\rangle,
\qquad
\langle u,L^{s}u\rangle\ge0 .
\]
Given the spectral decomposition $L=\sum_{k=1}^n\lambda_k\phi_k\phi_k^{\top}$
from Section~\ref{sec:spectral}, the fractional power
is~\cite{kato1995perturbation,samko1993fractional}
\[
L^{s}=\sum_{k=1}^n\lambda_k^{s}\phi_k\phi_k^{\top},
\]
so $L^{s}\phi_k=\lambda_k^{s}\phi_k$, and $\ker(L^{s})=\ker(L)$ for every
$s>0$. The fractional heat kernel admits the spectral representation
\[
(e^{-tL^{s}})_{ij}=\sum_{k=1}^n e^{-t\lambda_k^{s}}\phi_k(x_i)\phi_k(x_j).
\]
For $s\in(0,1)$ the fractional Sobolev space is
\[
H^{s}(V)=\bigl\{u\in\ell^2(V):\ \|L^{s/2}u\|_{\ell^2}<\infty\bigr\},
\qquad
\|u\|_{H^{s}}=\|u\|_{\ell^2}+\|L^{s/2}u\|_{\ell^2}.
\]

\subsection{Properties and Regularization}
The quadratic form of $L^{s}$ defines the fractional Dirichlet energy
\[
E_s(u)=\langle u,L^{s}u\rangle=\sum_{i=1}^n\lambda_i^{s}|\langle u,\phi_i\rangle|^2 .
\]
The fractional heat equation $u'=-L^{s}u$ is the gradient flow of $E_s/2$ in
the $\ell^2$ metric, since $\nabla E_s(u)=2L^{s}u$. A related one-step implicit
Euler, or Tikhonov, regularization is
\[
u^\star=\arg\min_u\left\{\tfrac12\|u-u_0\|_{\ell^2}^2+\tfrac{t}{2}E_s(u)\right\},
\qquad
u^\star=(I+tL^{s})^{-1}u_0 .
\]
Both $e^{-tL^{s}}u_0$ and $(I+tL^{s})^{-1}u_0$ suppress high-frequency modes,
through different spectral filters, and this variational principle is the sense
in which fractional diffusion offers $H^{s}$-regularization.

The order $s$ provides an additional hyperparameter, and the space-fractional
operator $L^{s}$ is nonlocal: for non-integer $s$ its entries do not vanish
beyond the graph neighborhood, whereas an integer power $L^{k}$ couples only
nodes within $k$ hops. Such nonlocal kernels are used to model anomalous
diffusion capturing both local clustering and long-range
dependencies~\cite{klicpera2019diffusion}, and the nonlocality has been
suggested to improve robustness to local perturbations of the
graph~\cite{berberidis2019adaptive}.  The precise effect of $s$ on the
spectral filter is the subject of the next subsection.

\subsection{How the Fractional Order Shapes the Spectral Filter}
\label{subsec:quantitative_fractional_order}
The action of $e^{-tL^{s}}$ on mode $\phi_k$ is multiplication by the scalar
filter $g_s(\lambda_k)=e^{-t\lambda_k^{s}}$. Writing the filter as a function of
the continuous variable $\lambda$ exposes how $s$ trades local against global
behavior. The following elementary comparison follows from the standard
spectral representation of functions of a symmetric graph
Laplacian~\cite{chung1997spectral}.

\begin{proposition}[Effect of the fractional order on spectral attenuation]
\label{prop:rotation}
Let $L$ be a fixed symmetric positive semidefinite graph Laplacian, and fix a
common diffusion time $t>0$. For $s>0$ define the spectral multiplier
\[
g_s(\lambda)=e^{-t\lambda^{s}},\qquad\lambda\ge0,
\]
so that $e^{-tL^{s}}\phi_k=g_s(\lambda_k)\phi_k$ for every eigenpair
$(\lambda_k,\phi_k)$ of $L$. For every $s>0$ the multipliers satisfy
$g_s(0)=1$ and $g_s(1)=e^{-t}$. For each fixed $\lambda\in(0,1)$ the map
$s\mapsto g_s(\lambda)$ is strictly increasing, whereas for each fixed
$\lambda>1$ it is strictly decreasing. Equivalently, for any $0<s_1<s_2$,
\[
\begin{cases}
g_{s_1}(\lambda)<g_{s_2}(\lambda), & 0<\lambda<1,\\[1mm]
g_{s_1}(\lambda)=g_{s_2}(\lambda), & \lambda\in\{0,1\},\\[1mm]
g_{s_1}(\lambda)>g_{s_2}(\lambda), & \lambda>1.
\end{cases}
\]
Consequently, at a fixed diffusion time, decreasing the fractional order
increases attenuation on positive-frequency modes below $\lambda=1$ and
decreases attenuation on modes above $\lambda=1$. The null-space component is
unchanged, and any mode with eigenvalue $\lambda=1$ has attenuation factor
$e^{-t}$ independently of the order. In particular, for $0<s<1$ these
comparisons hold relative to the classical heat kernel $e^{-tL}$; for $s>1$ the
strict inequalities relative to the classical kernel are reversed.

For every fixed $s>0$ the function $\lambda\mapsto g_s(\lambda)$ remains
strictly decreasing on $[0,\infty)$. Thus the fractional order controls the distribution of attenuation within
a family of low-pass filters.
\end{proposition}

\begin{proof}
The identities at $\lambda=0$ and $\lambda=1$ follow from $0^{s}=0$ and
$1^{s}=1$ for $s>0$. For $\lambda>0$, differentiation with respect to $s$ gives
\[
\frac{\partial g_s(\lambda)}{\partial s}
=-t\,\lambda^{s}\log(\lambda)\,e^{-t\lambda^{s}} .
\]
Since $t$, $\lambda^{s}$ and $e^{-t\lambda^{s}}$ are positive, the derivative
has the sign of $-\log\lambda$: positive for $0<\lambda<1$, zero at
$\lambda=1$, and negative for $\lambda>1$. This proves the monotonicity in $s$
and the stated comparisons. For fixed $s>0$ and $\lambda>0$,
\[
\frac{\partial g_s(\lambda)}{\partial\lambda}
=-ts\,\lambda^{s-1}e^{-t\lambda^{s}}<0,
\]
which together with $g_s(0)=1$ and $g_s(\lambda)<1$ for every $\lambda>0$
proves strict decrease on $[0,\infty)$.
\end{proof}

\begin{remark}\label{rem:unequal_diffusion_times}
For $0<s<1$ and positive times $T_s,T_1$, the positive-frequency crossing of
$e^{-T_s\lambda^{s}}$ and $e^{-T_1\lambda}$ is
\[
\lambda_*=\left(\frac{T_s}{T_1}\right)^{1/(1-s)} .
\]
Thus, the crossing equals one only when $T_s=T_1$, and comparisons using
independently selected diffusion times cannot be interpreted solely through the
equal-time crossing. The order is nevertheless not a reparametrization of the
time. For two positive eigenvalues $\lambda_a\neq\lambda_b$ the ratio of the
decay exponents of $e^{-T\lambda^{s}}$ at the two modes is
$(\lambda_b/\lambda_a)^{s}$, independent of $T$: rescaling the time multiplies
every exponent by the same factor and leaves this ratio unchanged, whereas the
order changes it. Matching the fractional multiplier to the classical one at
both modes would require $T_s/T_1=\lambda_a^{1-s}=\lambda_b^{1-s}$, which is
impossible for $s\neq1$, so no choice of $T_1$ reproduces $e^{-T_sL^{s}}$ on
the span of two eigenvectors with distinct positive eigenvalues. Each order defines a distinct one-parameter family of filters. 
\end{remark}

Proposition~\ref{prop:rotation} characterizes attenuation for a fixed
Laplacian normalization and a common diffusion time. For the symmetric
normalized Laplacian the spectrum lies in $[0,2]$~\cite{chung1997spectral},
and the graph and the class signal determine the spectral energy on either
side of the crossing. Since $0<g_s(\lambda)\le1$, each multiplier measures
the retention of the input mode. For source-driven propagation the spectral
coefficients of both the initial state and the source enter the response;
Section~\ref{sec:framework} develops that response, and
Section~\ref{sec:experiments} evaluates the associated classification
behavior.

To make Proposition~\ref{prop:rotation} quantitative, we write the ratio of the
fractional filter to the classical one at a fixed eigenvalue,
\[
R_s(\lambda,t)=\frac{g_s(\lambda)}{g_1(\lambda)}
=\exp\!\bigl(-t(\lambda^{s}-\lambda)\bigr),
\qquad\text{so}\qquad
\log R_s(\lambda,t)=-t\bigl(\lambda^{s}-\lambda\bigr).
\]
For $0<s<1$ and $0<\lambda<1$ one has $\lambda^{s}>\lambda$, so $R_s<1$ and the
ratio decays exponentially in $t$: the modes below the fixed point are damped
more strongly than by the classical kernel. For $\lambda>1$ one has
$\lambda^{s}<\lambda$, so $R_s>1$ and the ratio grows exponentially in $t$: the
modes above the fixed point are retained more strongly. The exponent $t|\lambda^{s}-\lambda|$ is $t$ times the difference between
the decay rates of the two kernels at the same mode, and $R_s$ measures
relative retention: both multipliers tend to zero as $t\to\infty$ for every
$\lambda>0$, so the ratio can grow while the difference of the multipliers
vanishes.

Table~\ref{tab:filter} reports the filter at the spectral extremes and at the
fixed point for a representative graph, a two-block stochastic block model with
$300$ nodes, symmetric normalized Laplacian and $t=5$. The crossing value
$g_s(1)=e^{-5}\approx0.0067$ is independent of $s$, while $g_s(\lambda_2)$
rises from $0.089$ at $s=0.3$ to $0.876$ at $s=1.5$ and $g_s(\lambda_{\max})$
falls over the same range.

\begin{table}[htbp]
\centering
\caption{Spectral filter $g_s(\lambda)=e^{-t\lambda^{s}}$ at $t=5$ on the
symmetric normalized Laplacian of a two-block stochastic block model
($\lambda_2=0.089$, $\lambda_{\max}=1.40$). The fixed point $g_s(1)=e^{-5}$ is
independent of $s$.}
\label{tab:filter}
\begin{tabular}{c|ccc}
\toprule
$s$ & $g_s(\lambda_2)$ & $g_s(1)$ & $g_s(\lambda_{\max})$ \\
\midrule
0.30 & 0.089 & 0.0067 & 0.0040 \\
0.50 & 0.226 & 0.0067 & 0.0027 \\
0.75 & 0.444 & 0.0067 & 0.0016 \\
1.00 & 0.642 & 0.0067 & 0.0009 \\
1.50 & 0.876 & 0.0067 & 0.0003 \\
\bottomrule
\end{tabular}
\end{table}

For truncation, the fractional and classical cases differ only through the
exponent of the cut eigenvalue,
\[
\Bigl\|e^{-tL^{s}}u-\textstyle\sum_{k=1}^m e^{-t\lambda_k^{s}}
\langle u,\phi_k\rangle\phi_k\Bigr\|_2\le e^{-t\lambda_{m+1}^{s}}\|u\|_2 .
\]
For $0<s<1$ and $\lambda_{m+1}>1$ the fractional bound is larger than the
classical bound at the same cutoff, so the fractional kernel has a larger
worst-case truncation error and may require more retained modes to meet a
prescribed tolerance. Both quantities are worst-case bounds over unit signals:
either may already meet a given tolerance at the same rank, and both vanish for
a signal supported on the retained eigenspace.

Whether this spectral redistribution improves accuracy under sparse supervision,
and how it affects over-smoothing, is tested empirically in
Section~\ref{subsec:fractional_order_sensitivity}.

\paragraph{Numerical validation of spectral damping.}
To validate the damping behavior predicted by the fractional multiplier, we
compute eigenpairs of the symmetric normalized Laplacian $L_{\mathrm{sym}}$ on
Cora, $L_{\mathrm{sym}}\phi_k=\lambda_k\phi_k$ with $\|\phi_k\|_2=1$. For an
initial graph signal $U_0$, taken here as the node-feature matrix, the
projection onto the eigenbasis is $\widehat U_k(0)=\phi_k^{\top}U_0$, and under
$U(T)=e^{-TL_{\mathrm{sym}}^{s}}U_0$ the coefficients evolve as
$\widehat U_k(T)=e^{-T\lambda_k^{s}}\widehat U_k(0)$, so the damped amplitude is
$\|\widehat U_k(T)\|_2=e^{-T\lambda_k^{s}}\|\phi_k^{\top}U_0\|_2$.
Figure~\ref{fig:cora_spectral_damping_combined} evaluates the ratio
$R_{1/2}(\lambda,T)$ of the preceding subsection on the Cora spectrum. The sign
change at $\lambda=1$ is visible in the sampled eigenvalues and confirms that,
relative to the classical kernel, the fractional kernel with $s=0.5$ damps
modes with $\lambda_k<1$ more strongly and modes with $\lambda_k>1$ more
weakly. 

\begin{figure}[t]
\centering
\includegraphics[width=\linewidth]{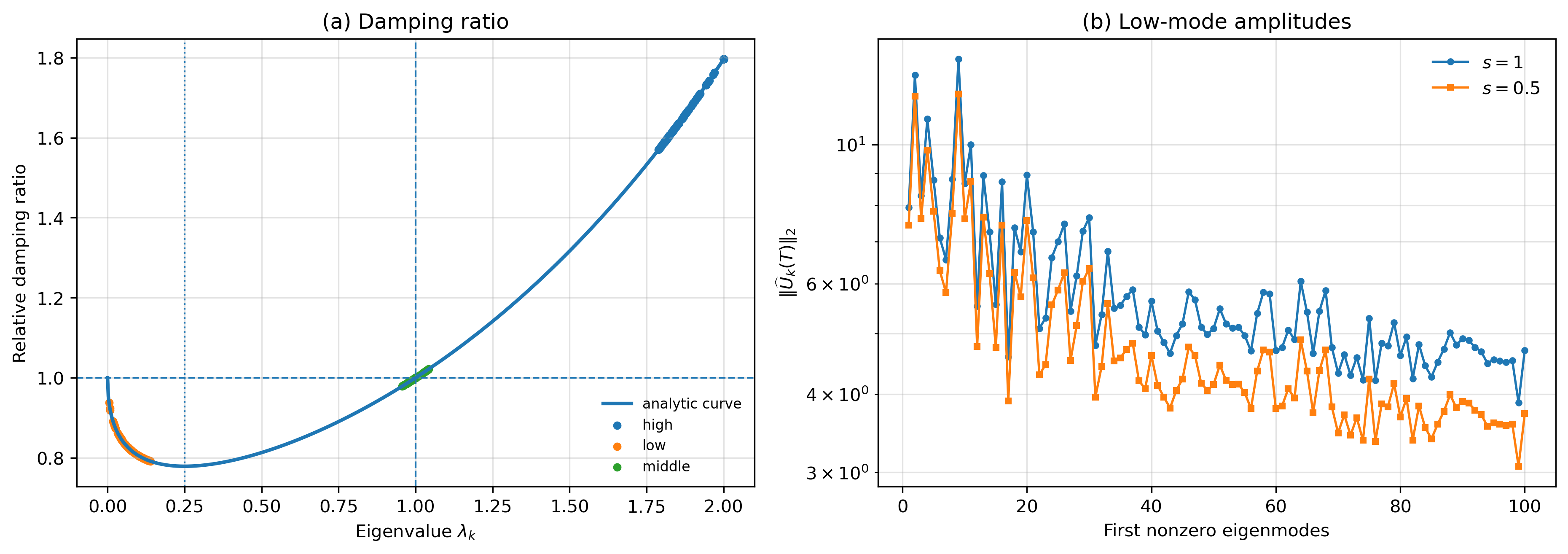}
\caption{Spectral damping validation on Cora. Panel~(a) plots the relative
damping ratio $R(\lambda_k)=e^{-T\lambda_k^{0.5}}/e^{-T\lambda_k}$, with sampled
eigenvalues overlaid on the analytic curve. Panel~(b) compares damped
low-frequency spectral amplitudes for $s=1$ and $s=0.5$. The transition at
$\lambda_k=1$ verifies that fractional diffusion damps modes below one more
strongly and modes above one more weakly than classical heat diffusion.}
\label{fig:cora_spectral_damping_combined}
\end{figure}

\subsection{Computational Methods and Error Analysis}\label{sec:fraccomp}
The truncation error of the fractional kernel was bounded at the end of
Section~\ref{subsec:quantitative_fractional_order}. Its propagator is also
Lipschitz in time: since $L^{s}$ is a bounded symmetric matrix with spectrum
in $[0,\lambda_{\max}^{s}]$, the mean value theorem applied to
$x\mapsto e^{-ax}$ on that interval gives
\[
\bigl\|e^{-aL^{s}}-e^{-bL^{s}}\bigr\|_2\le\lambda_{\max}^{s}\,|a-b|
\qquad\text{for all }a,b\ge0,
\]
so the variation-of-constants integrals of Section~\ref{sec:framework} are
well defined and the exponential-integrator steps of
Section~\ref{sec:framework} are stable with respect to the step length.

The Chebyshev framework of Section~\ref{sec:spectral} also applies to the
fractional heat kernel. In the classical case one approximates the scalar
filter $x\mapsto e^{-tx}$; in the fractional case it is
\[
f_s(x)=e^{-tx^{s}},\qquad x\in[0,\lambda_{\max}].
\]
The same recurrence can be applied, but the convergence rate changes. When $s$
is a positive integer, $f_s$ is entire and the Chebyshev expansion converges
geometrically~\cite{trefethen2000spectral}. When $s$ is non-integer, $x^{s}$ has
a branch-point singularity at $x=0$, an endpoint of $[0,\lambda_{\max}]$. This
loss of analyticity prevents the usual Bernstein-ellipse argument and leaves
algebraic rather than geometric decay, consistent with the known numerical
difficulty of fractional matrix functions and related non-integer
powers~\cite{hale2008computing}. For a function with an algebraic endpoint
singularity of exponent $s$, the classical coefficient analysis gives shifted
Chebyshev coefficients of order $j^{-2s-1}$ and hence
\[
\bigl\|e^{-t(\cdot)^{s}}-p_m\bigr\|_{\infty,[0,\lambda_{\max}]}=O\!\left(m^{-2s}\right),
\qquad
\bigl\|e^{-tL^{s}}-p_m(L)\bigr\|_2=O\!\left(m^{-2s}\right),
\]
with constants that need not be uniform in $s$, $t$ and $\lambda_{\max}$.

Table~\ref{tab:cheb} illustrates this at $t=5$ on the graph of
Table~\ref{tab:filter}. For $s=1$, degree $m=8$ already gives operator-norm
error $4.2\times10^{-5}$ and $m=20$ reaches machine precision. For $s=0.75$,
degree $m=32$ gives $5.0\times10^{-3}$, and for $s=0.5$ the error at the same
degree is still $5.8\times10^{-2}$. The measured decay is consistent with the
rate $m^{-2s}$, which predicts a reduction by $(m_2/m_1)^{2s}$ when the degree
is raised from $m_1$ to $m_2$: between $m=8$ and $m=32$ the error falls by
$3.6$ at $s=0.5$ against the predicted $4$, and by $8.6$ at $s=0.75$ against the
predicted $8$; between $m=16$ and $m=32$ it falls by $1.9$ and $2.8$ against
the predicted $2$ and $2.8$.

\begin{table}[htbp]
\centering
\caption{Operator-norm Chebyshev error $\|e^{-tL^{s}}-p_m(L)\|_2$ at $t=5$,
where $p_m$ approximates $x\mapsto e^{-tx^{s}}$ on the spectrum of $L$.
Convergence is geometric at $s=1$ and only algebraic for non-integer $s<1$.}
\label{tab:cheb}
\begin{tabular}{c|ccc}
\toprule
$m$ & $s=0.5$ & $s=0.75$ & $s=1.0$ \\
\midrule
4  & $3.7\times10^{-1}$ & $1.3\times10^{-1}$ & $1.8\times10^{-2}$ \\
8  & $2.1\times10^{-1}$ & $4.3\times10^{-2}$ & $4.2\times10^{-5}$ \\
12 & $1.5\times10^{-1}$ & $2.2\times10^{-2}$ & $2.0\times10^{-8}$ \\
16 & $1.1\times10^{-1}$ & $1.4\times10^{-2}$ & $3.0\times10^{-12}$ \\
20 & $9.1\times10^{-2}$ & $1.0\times10^{-2}$ & $2.1\times10^{-15}$ \\
24 & $7.7\times10^{-2}$ & $7.7\times10^{-3}$ & $2.3\times10^{-15}$ \\
32 & $5.8\times10^{-2}$ & $5.0\times10^{-3}$ & $2.9\times10^{-15}$ \\
\bottomrule
\end{tabular}
\end{table}

For fractional orders we therefore note two alternatives to the Chebyshev
expansion, without claiming a convergence rate for either. For $s=1/2$ the
Bochner subordination formula~\cite{berg1984harmonic,schilling2012bernstein}
\[
e^{-t\sqrt{L}}u
=\frac{t}{2\sqrt\pi}\int_0^\infty\tau^{-3/2}e^{-\frac{t^2}{4\tau}}\,
e^{-\tau L}u\,d\tau
\]
reduces the fractional kernel to an average of classical heat kernels. The
accuracy then depends on the quadrature rule for the $\tau$-integral and on the
treatment of the null space: since $e^{-\tau L}$ is the identity on $\ker(L)$,
truncating the integral at $\tau=R$ leaves
$\operatorname{erf}\bigl(t/(2\sqrt R)\bigr)\approx t/\sqrt{\pi R}$ on a null
vector, an algebraic tail. Applying the formula to $(I-\Pi_0)u$ only and adding
$\Pi_0u$ exactly removes that term, after which the tail is bounded by
$e^{-R\lambda_+}$ and is small once $R\gg1/\lambda_+$. A rational Krylov
approximation of $L^{s}$ places its poles off the spectrum and is the standard
approach for non-integer matrix powers~\cite{higham2008functions}. Neither
alternative is analyzed or used here; the experiments use the exact
eigendecomposition where it is feasible and the Chebyshev expansion otherwise.

\section{Integrating the Heat Kernel in Graph Neural Networks}
\label{sec:gnn}

The fractional heat kernel can be used as a graph propagation operator inside standard neural architectures. In contrast to ordinary message passing, which aggregates information only from immediate neighbors at each layer, the heat-kernel operator propagates information through the graph according to the semigroup generated by the fractional Laplacian. In this section, we describe the source-free neural component of the method. The source-driven semi-supervised propagation model is introduced in Section~\ref{sec:framework}, and the two components are combined in Algorithm~\ref{alg:gat-heat-kernel}.

Throughout this section, we write
\[
K_s(t)=e^{-tL^s},
\]
where \(L\) is a symmetric graph Laplacian, \(s>0\) is the fractional order, and \(t>0\) is the diffusion time. Unless otherwise stated, the heat-kernel GNN layers in this section use the symmetric normalized Laplacian \(L_{\rm sym}\) of the original input graph.  In this section, the heat-kernel layer is described for a fixed graph Laplacian. The distinction between original-graph diffusion, GAT-refined graph diffusion, and neural source-driven diffusion is made later, after the source-driven framework has been introduced.

The layers in this section are used in the transductive semi-supervised node-classification setting. All node features and graph edges are used during message passing, but the supervised loss is evaluated only on the labeled nodes. Thus, the heat-kernel layers should be viewed as the homogeneous, source-free component of the full source-driven framework developed in Section~\ref{sec:framework}.
 The classical heat kernel was used as a fixed low-pass graph convolutional
filter in GraphHeat~\cite{xu2020graph}. Its layer retains the first two terms
of a heat-kernel polynomial and therefore carries two independent weight
matrices,
\[
H^{(\ell+1)}
=
\sigma\!\left(
H^{(\ell)}W_0^{(\ell)}
+
\widetilde K\,H^{(\ell)}W_1^{(\ell)}
\right),
\qquad
\widetilde K=\mathcal T_\epsilon\!\left(e^{-tL}\right),
\]
where $\mathcal T_\epsilon$ sets to zero every entry of magnitude at most
$\epsilon$, and $L=I-D^{-1/2}AD^{-1/2}$ is formed without self-loops. The
layer of Section~\ref{subsec:heat_gcn} has a single weight matrix, so at
$s=1$ it is a heat-kernel-filtered graph convolution and not GraphHeat's
layer, which adds an identity channel with an independent weight matrix.
Section~\ref{subsec:graphheat_controlled} evaluates the original GraphHeat
architecture and states the Laplacian and sparsification conventions used
in each comparison. The scale parameter written $s$ in~\cite{xu2020graph}
is the diffusion time $t$ here.
 
\subsection{Heat Kernel in Graph Convolutional Networks}
\label{subsec:heat_gcn}

A standard GCN layer applies~\cite{KipfWelling2017GCN}
\[
H^{(\ell+1)}
=
\sigma\!\left(
\tilde D^{-1/2}\tilde A\tilde D^{-1/2}
H^{(\ell)}W^{(\ell)}
\right),
\]
where \(\tilde A=A+I\), \(\tilde D\) is the corresponding degree matrix, \(H^{(\ell)}\in\mathbb{R}^{n\times d_\ell}\) is the matrix of node representations at layer \(\ell\), and \(W^{(\ell)}\) is a trainable weight matrix. Replacing the normalized adjacency operator by the fractional heat kernel gives
\[
H^{(\ell+1)}
=
\sigma\!\left(
K_s(t_\ell)H^{(\ell)}W^{(\ell)}
\right)
=
\sigma\!\left(
e^{-t_\ell L^s}H^{(\ell)}W^{(\ell)}
\right).
\]
The parameter \(t_\ell\) controls the diffusion scale at layer \(\ell\), while \(s\) controls the spectral profile of the propagation. For an eigenmode of \(L\) with eigenvalue \(\lambda_k\), the filter response is
\[
g_s(\lambda_k,t_\ell)=e^{-t_\ell\lambda_k^s}.
\]
Thus the layer applies a smooth spectral low-pass filter before the trainable channel mixing. When \(s=1\), this is the classical heat-kernel filter. When \(0<s<1\), the filter changes the relative damping of low- and high-frequency modes, as analyzed in Section~\ref{subsec:quantitative_fractional_order}.

When \(L\), \(s\), and \(t_\ell\) are fixed, \(K_s(t_\ell)\) is a fixed linear operator, so training proceeds as in a standard GCN with a fixed propagation matrix. If
\[
A^{(\ell+1)}=K_s(t_\ell)H^{(\ell)}W^{(\ell)}
\]
denotes the pre-activation matrix and \(G^{(\ell+1)}=\partial\mathcal{L}/\partial A^{(\ell+1)}\), then
\[
\frac{\partial\mathcal{L}}{\partial W^{(\ell)}}
=
\left(K_s(t_\ell)H^{(\ell)}\right)^\top G^{(\ell+1)}.
\]
The heat-kernel replacement needs no special backpropagation rule for the trainable weights, and the stabilization is the standard one: dropout, weight decay and validation-based early stopping.

\subsection{Multi-Scale Heat Kernel Aggregation}
\label{subsec:multiscale_heat}

A single diffusion time may not capture all relevant graph scales. We   consider a multi-scale heat-kernel layer of the form
\[
H^{(\ell+1)}
=
\sigma\!\left(
\sum_{r=1}^{R}
\alpha_r^{(\ell)}
e^{-t_r^{(\ell)}L^s}
H^{(\ell)}W_r^{(\ell)}
\right),
\]
where \(t_r^{(\ell)}>0\) are diffusion times, \(W_r^{(\ell)}\) are trainable channel-mixing matrices, and \(\alpha_r^{(\ell)}\) are trainable scale weights. The weights are parametrized by unconstrained logits \(\eta_r^{(\ell)}\) through
\[
\alpha_r^{(\ell)}
=
\frac{\exp(\eta_r^{(\ell)})}
{\sum_{q=1}^{R}\exp(\eta_q^{(\ell)})},
\qquad
\sum_{r=1}^{R}\alpha_r^{(\ell)}=1.
\]
The logits \(\eta_r^{(\ell)}\) are optimized jointly with the remaining neural parameters using the same labeled-node cross-entropy loss. The softmax parametrization keeps the weights nonnegative and normalized, which prevents unbounded rescaling of the propagated features.

With fixed diffusion times the layer is again a differentiable layer with fixed propagation operators. When the times are learned, positivity is enforced by
\[
t_r^{(\ell)}=\operatorname{softplus}(\rho_r^{(\ell)}).
\]
Since \(L^s\) commutes with \(e^{-tL^s}\), the derivative of the propagated features with respect to \(t\) is
\[
\frac{\partial}{\partial t}
\left(e^{-tL^s}X\right)
=
-L^s e^{-tL^s}X,
\]
which in a Chebyshev implementation is evaluated by approximating the scalar derivative filter
\[
q_{s,t}(\lambda)=-\lambda^s e^{-t\lambda^s}
\]
on the spectrum of \(L\).
\section{Proposed Heat Kernel Diffusion Framework}\label{sec:framework}
We motivate and provide algorithmic details for a heat kernel propagation method for semi-supervised graph learning. We first present a continuous model giving the intuition, then detail the algorithms we implement.
\subsection{Forward Propagation with Continuous Supervision}
Let $V = \{x_1, \dots, x_n\}$ be the vertices. A subset $V_l = \{x_1, \dots, x_l\} \subset V$ have labels and form the training set $(x_i, y_i)_{i=1}^{l}$. The goal is to extend the labels from $V_l$ to the unlabeled set $V_u = \{x_{l+1}, \dots, x_n\}=V\setminus V_l$. Let $U_0\in\mathbb R^{n\times c}$ be the initial label matrix, where $c$ is the number of classes: row $U_0(i,:)$ is the one-hot label of $x_i$ for $i\le l$ and zero otherwise.

We consider the system
\begin{equation}\label{eq:heat_diffusion}
\frac{dU}{dt} = -L^s U + F, \quad U(0) = U_0, \quad s>0,
\end{equation}
where $L^s$ is the fractional power of the chosen graph Laplacian, in the sense of Section~\ref{sec:fractional}, defined for every $s>0$. The regime $0<s<1$ is the nonlocal, long-range space-fractional regime. Proposition~\ref{prop:rotation} describes the effect of 
$s$ on the filter.  

GRAND++~\cite{thorpe2022grandpp} established source-driven graph diffusion for learning at low labeling rates, with a learned attention operator and a random-walk-derived source, and proved that its trajectories do not collapse to a constant as $t\to\infty$. Our formulation develops this approach for the spectral fractional Laplacian $L^s$ with a source centered against its full null space: Theorem~\ref{thm:bounded} is the corresponding non-degeneracy statement, the analysis characterizes the conserved label functionals and the source-driven equilibria, and the algorithms combine the propagation with learned graph refinement and optional self-training.

The source term $F$ is chosen so that it injects supervision only through the labeled nodes and, at the same time, remains orthogonal to the null space of the Laplacian used in the diffusion. This condition preserves the corresponding conserved label functionals and prevents secular drift in the source-driven solution.

Let $C_1,\dots,C_r$ be the connected components of the graph used in the
diffusion, and let $\phi\in\mathbb R^n$ be a vector spanning $\ker(L)$
componentwise, with positive entries. For the combinatorial Laplacian
$\phi=\mathbf 1$; for the symmetric normalized Laplacian with positive
vertex degrees $\phi=D^{1/2}\mathbf 1$. Since $\ker(L^s)=\ker(L)$ for
$s>0$, the restrictions $\phi|_{C_1},\dots,\phi|_{C_r}$, normalized, form
an orthonormal basis $Q$ of $\ker(L^s)$, and $\Pi_0=QQ^\top$.

For a labeled set $\mathcal L$ write $\mathcal L_a=\mathcal L\cap C_a$.
The source is centered within each component that contains a labeled
node,
\[
\bar U_{\phi,a}
=
\frac{\sum_{i\in\mathcal L_a}\phi_i\,U_0(i,:)}
     {\sum_{i\in\mathcal L_a}\phi_i},
\qquad
F(i,:)
=
\begin{cases}
U_0(i,:)-\bar U_{\phi,a}, & i\in\mathcal L_a,\\[1mm]
0, & \text{otherwise},
\end{cases}
\]
and $F$ vanishes identically on every component containing no labeled
node. We write $F=\mathcal C_{\phi}(U_0,\mathcal L)$ for this operator.
By construction $\sum_{i\in C_a}\phi_iF(i,:)=0$ for every $a$, hence
\[
Q^\top F=0,
\qquad\text{equivalently}\qquad
\Pi_0F=0 .
\]
Global centering against a single vector gives only
$\phi^\top F=0$, which is strictly weaker when $r>1$ and does not
imply $\Pi_0F=0$. The same operator applies to any Laplacian used in
the diffusion; when diffusion is performed on a refined Laplacian $L'$
the components and the null vector are those of $L'$. This componentwise
definition of $\mathcal C_\phi$ is the one used throughout the paper, in
Section~\ref{subsec:optional_selftraining}, in
Section~\ref{subsec:neural_fhkpp} and in the algorithms.

\begin{remark}\label{rem:components}
A component containing exactly one labeled node has
$U_0(i,:)=\bar U_{\phi,a}$, so its source vanishes; under very sparse
supervision this removes the source term from many small components.
A component containing no labeled node receives no supervision at all
and retains its initial null-space content for every $t$.  Both effects come from the compatibility requirement and are independent of the order.
\end{remark}

The exact solution of~\eqref{eq:heat_diffusion} is
\begin{equation}
U(t) = e^{-tL^s} U_0 + \int_0^t e^{-(t-\tau)L^s} F \, d\tau = e^{-tL^s} U_0 + \int_0^t e^{-\tau L^s} F \, d\tau .
\end{equation}
For time-independent $F$ we evaluate the integral. Since $L^s$ is singular, we use the series expansion $e^{-\tau L^s} = \sum_{k=0}^\infty \frac{(-\tau L^s)^k}{k!}$ and integrate term by term:
\begin{equation}\label{eq:series_integral}
\int_0^t e^{-\tau L^s} d\tau = t \sum_{k=0}^\infty \frac{(-tL^s)^k}{(k+1)!} = t \cdot h(tL^s),
\end{equation}
where
\[
h(x) = \sum_{k=0}^\infty \frac{(-x)^k}{(k+1)!} = \frac{1-e^{-x}}{x},
\]
with $h(0) = 1$ by continuity. The complete solution is
\begin{equation*}
U(t) = e^{-tL^s} U_0 + t \cdot h(tL^s) F.
\end{equation*}
\subsection{Theoretical Properties}
We first fix the projector onto the kernel of $L^s$.
\begin{definition}\label{def:projector}
Let $\Pi_0$ be the orthogonal projector onto $\ker(L)$, which equals
$\ker(L^s)$ for every $s>0$, and let $r=\dim\ker(L)$. For the combinatorial
Laplacian $L=D-W$, and for the symmetric normalized Laplacian when every
vertex degree is positive, $r$ is the number of connected
components~\cite{chung1997spectral}. When $r=1$ the projector is rank one:
\begin{itemize}
  \item combinatorial Laplacian: $\ker(L)=\operatorname{span}\{\mathbf 1\}$ and
        $\Pi_0=\tfrac1n\mathbf 1\mathbf 1^\top$;
  \item symmetric normalized Laplacian:
        $\ker(L_{\mathrm{sym}})=\operatorname{span}\{D^{1/2}\mathbf 1\}$ and
        $\Pi_0=\dfrac{D^{1/2}\mathbf 1\,\mathbf 1^\top D^{1/2}}{\mathbf 1^\top D\mathbf 1}$.
\end{itemize}
When $r>1$ the same formulas hold componentwise. With $\phi=\mathbf 1$,
respectively $\phi=D^{1/2}\mathbf 1$, and $\phi|_{C_a}$ the restriction of
$\phi$ to the component $C_a$ extended by zero,
\[
\Pi_0=QQ^\top=\sum_{a}
\frac{\phi|_{C_a}\,\phi|_{C_a}^{\top}}{\|\phi|_{C_a}\|_2^{2}},
\]
where $Q$ is the orthonormal basis of $\ker(L)$ introduced in
Section~\ref{sec:framework} and the sum runs over the components on which
$\phi$ is nonzero; the projector does not depend on the choice of orthonormal
basis. We write $\lambda_+$ for the smallest positive eigenvalue of $L$, which
equals $\lambda_2$ only when the graph is connected.

The citation graphs used in Section~\ref{sec:experiments} are not connected.
Cora has $78$ components and no isolated vertices, so
$\dim\ker(L_{\mathrm{sym}})=78$ under either self-loop convention. CiteSeer
has $438$ components, of which $48$ are isolated vertices; without self-loops
such a vertex has $d_i=0$, the convention $d_i^{-1/2}=0$ makes the
corresponding row of $L_{\mathrm{sym}}$ equal to $e_i$, and it contributes
the eigenvalue one rather than zero, so $\dim\ker(L_{\mathrm{sym}})=390$.
With self-loops the degree is one, the row vanishes, and the nullity equals
the component count, $438$. The rank-one formulas   do not apply to
either graph, and $\Pi_0$ is used throughout.
\end{definition}

\begin{theorem}[Conserved label functional]\label{thm:mass}
Let $L$ be a symmetric positive semidefinite graph Laplacian, let $Q$ be an
orthonormal basis of $\ker(L)=\ker(L^s)$ with $\Pi_0=QQ^\top$ the projector
of Definition~\ref{def:projector}, and let $U(t)$ solve~\eqref{eq:heat_diffusion}
with initial state $U(0)=U_0$ and constant source $F$. Then, for all $t\ge0$,
\[
Q^\top U(t) = Q^\top U_0 + t\,Q^\top F,
\qquad
\Pi_0 U(t)=\Pi_0U_0+t\,\Pi_0F .
\]
In particular $Q^\top U(t)$ is constant in $t$ if and only if $\Pi_0F=0$.
When the graph is connected, $Q=\varphi_1$ is the single unit null vector,
equal to $\mathbf 1/\sqrt n$ for the combinatorial Laplacian and to
$D^{1/2}\mathbf 1/\|D^{1/2}\mathbf 1\|_2$ for the symmetric normalized
Laplacian, and the identity reduces to a scalar statement for each feature
column.
\end{theorem}

\begin{proof}
Since $L^s$ is symmetric and $L^sQ=0$, we have $Q^\top L^s=(L^sQ)^\top=0$.
Differentiating along the solution of~\eqref{eq:heat_diffusion},
\[
\frac{d}{dt}\bigl(Q^\top U(t)\bigr)
= Q^\top\bigl(-L^sU(t)+F\bigr)
= Q^\top F,
\]
and integrating from $0$ to $t$ gives the first identity. Multiplying it on
the left by $Q$ and using $\Pi_0=QQ^\top$ gives the second. Because the
columns of $Q$ are orthonormal, $Q^\top F=0$ if and only if
$QQ^\top F=\Pi_0F=0$, which is the conservation statement.
\end{proof}

\begin{remark}\label{rem:masspractice}
The conserved functional depends on the Laplacian and on the dimension of its
null space. For the combinatorial Laplacian $L=D-W$ on a connected graph it is
$\mathbf 1^\top U(t)$, and for the symmetric normalized Laplacian it is
$(D^{1/2}\mathbf 1)^\top U(t)$. On a disconnected graph these single functionals
are replaced by the $r$ componentwise functionals collected in $Q^\top U(t)$ in
Theorem~\ref{thm:mass}. Ordinary arithmetic centering is therefore appropriate
only for the combinatorial Laplacian on a connected graph. In all experiments
using $L_{\mathrm{sym}}$ we use the degree-weighted componentwise source of
Section~\ref{sec:framework}, which satisfies $\Pi_0F=0$. If diffusion is
performed on a refined Laplacian $L'$, the source must be centered against
$\ker(L')$, using $\phi'=D'^{1/2}\mathbf 1$ and the components of $G'$, not
against the null space of the original graph Laplacian.
\end{remark}

\begin{theorem}\label{thm:bounded}
Let $G$ be a graph with $n$ nodes and let $L$ be a symmetric positive
semidefinite graph Laplacian, with $r=\dim\ker(L)\ge 1$ and $\Pi_0$ the
orthogonal projector onto $\ker(L^s)=\ker(L)$ of
Definition~\ref{def:projector}. Consider
\begin{equation}\label{eq:heat-source-2}
\frac{d}{dt}U(t) = -\,L^s U(t) + F,\qquad U(0)=U_0,
\end{equation}
with constant source $F\in\mathbb{R}^{n\times c}$, and let $(L^s)^\dagger$ be
the Moore-Penrose pseudoinverse of $L^s$. Then:
\begin{enumerate}
\item[(i)] The unique mild solution is
\[
U(t) = e^{-tL^s}U_0 + \int_0^t e^{-(t-\tau)L^s} F\,d\tau
     = e^{-tL^s}U_0 + t\,h(tL^s)F,
\]
with $h(x)=(1-e^{-x})/x$ and $h(0)=1$.
\item[(ii)] The solution remains bounded as $t\to\infty$ if and only if
$\Pi_0F=0$. In that case,
\[
\lim_{t\to\infty} U(t)=\Pi_0U_0+(L^s)^\dagger F,
\]
which is the unique steady state whose projection onto $\ker(L^s)$ equals
$\Pi_0U_0$ and whose range component is the minimum-norm particular solution
$(L^s)^\dagger F$.
\end{enumerate}
No connectivity assumption is used; when $r=1$ the projector is rank one and the
statement reduces to the connected case.
\end{theorem}

\begin{proof}
Order the eigenvalues so that $\lambda_1=\dots=\lambda_r=0$ and
$0<\lambda_{r+1}\le\dots\le\lambda_n$, and write $L^s=\Phi\Lambda\Phi^\top$ with
\[
\Lambda=\mathrm{diag}\bigl(\underbrace{0,\dots,0}_{r},
      \lambda_{r+1}^s,\dots,\lambda_n^s\bigr),
\qquad
\Phi=[\varphi_1,\dots,\varphi_n],
\qquad
Q=[\varphi_1,\dots,\varphi_r],
\]
so that $\Pi_0=QQ^\top$ and $L^sQ=0$. Every eigenvalue $\lambda_k^s$ with $k>r$
is positive, which is all the proof uses.

(i) Since $L^s$ is symmetric positive semidefinite, $e^{-tL^s}$ is a contraction
semigroup, and the variation-of-constants formula gives the integral form. Using
$\int_0^t e^{-\tau L^s}\,d\tau = t\,h(tL^s)$ yields the stated expression.

(ii) We separate the null-space and range components. Apply $\Pi_0$ to the
solution in (i). Because $L^s\Pi_0=0$ and $h(0)=1$, we have
$\Pi_0e^{-tL^s}=\Pi_0$ and $\Pi_0\,h(tL^s)=\Pi_0$, hence
\begin{equation}\label{eq:bounded-null}
\Pi_{0}U(t)=\Pi_{0}U_0 + t\,\Pi_{0}F .
\end{equation}
This is the projector form of Theorem~\ref{thm:mass}. The complementary
component is
\[
(I-\Pi_{0})U(t)=e^{-tL^s}(I-\Pi_{0})U_0
   +\int_0^t e^{-\tau L^s}(I-\Pi_{0})F\,d\tau .
\]
On $\mathrm{range}(I-\Pi_0)$ every eigenvalue $\lambda_k^s$ with $k>r$ is
positive, so
\[
\bigl\|e^{-tL^s}(I-\Pi_{0})U_0\bigr\|_F
\le e^{-t\lambda_+^{s}}\bigl\|(I-\Pi_{0})U_0\bigr\|_F
\longrightarrow 0,
\]
where $\lambda_+=\lambda_{r+1}$ is the smallest positive eigenvalue of $L$, and
\[
\int_0^t e^{-\tau L^s}(I-\Pi_{0})F\,d\tau
= \Phi\,\mathrm{diag}\!\Bigl(\underbrace{0,\dots,0}_{r},
   \tfrac{1-e^{-t\lambda_{r+1}^s}}{\lambda_{r+1}^s},\dots,
   \tfrac{1-e^{-t\lambda_n^s}}{\lambda_n^s}\Bigr)\Phi^\top F
\longrightarrow (L^s)^\dagger F .
\]
Thus $(I-\Pi_{0})U(t)$ is bounded and convergent for every $F$. By
\eqref{eq:bounded-null}, $\Pi_{0}U(t)$ is bounded if and only if $\Pi_{0}F=0$,
and since the two components are orthogonal,
$\|U(t)\|_F^2=\|\Pi_0U(t)\|_F^2+\|(I-\Pi_0)U(t)\|_F^2$, so $U(t)$ is bounded if
and only if $\Pi_{0}F=0$.

When $\Pi_{0}F=0$, the secular term in \eqref{eq:bounded-null} vanishes, so
$\Pi_{0}U(t)=\Pi_{0}U_0$ for all $t$. Combining with the limit of the
complementary component gives
\[
\lim_{t\to\infty}U(t)=\Pi_{0}U_0+(L^s)^\dagger F .
\]
This is the minimum-norm solution of $L^s U=F$ subject to $\Pi_{0}U=\Pi_{0}U_0$,
by the defining property of the pseudoinverse.
\end{proof}

\begin{definition}[Over-smoothing index]\label{def:osi}
Let $L$ be a symmetric positive semidefinite graph Laplacian and let
$\Pi_0$ be the projector of Definition~\ref{def:projector} onto $\ker(L)$.
For node representations $Z_0\in\mathbb R^{n\times d_0}$ and
$Z_T\in\mathbb R^{n\times d_T}$ with $(I-\Pi_0)Z_0\neq0$, define the
over-smoothing index
\[
\operatorname{OSI}(T)
=
\frac{\|(I-\Pi_0)Z_T\|_F}{\|(I-\Pi_0)Z_0\|_F}.
\]
When $Z_T$ is a propagated copy of $Z_0$, so that $d_T=d_0$, this is the
fraction of the variation orthogonal to $\ker(L)$ that survives the
propagation. Values near zero indicate collapse onto the null space of the
Laplacian, that is, over-smoothing.
\end{definition}

\begin{proposition}[Spectral decay and over-smoothing]\label{prop:osi}
Let $r=\dim\ker(L)$, order the eigenvalues of $L$ as
$\lambda_1=\dots=\lambda_r=0<\lambda_{r+1}\le\cdots\le\lambda_n$, and
let $\lambda_+=\lambda_{r+1}$ be the smallest positive eigenvalue, as in
Definition~\ref{def:projector}; the hypothesis $(I-\Pi_0)Z_0\neq0$
forces $r<n$, so $\lambda_+$ exists. For the homogeneous fractional heat
flow
\[
Z(T)=e^{-TL^s}Z_0,
\]
one has
\[
\|(I-\Pi_0)Z(T)\|_F
\le
e^{-T\lambda_+^s}\|(I-\Pi_0)Z_0\|_F.
\]
Consequently,
\[
\operatorname{OSI}(T)\le e^{-T\lambda_+^s}.
\]
\end{proposition}

\begin{proof}
Let $\phi_1,\dots,\phi_n$ be an orthonormal eigenbasis of $L$ ordered as
in the statement, so that $\Pi_0=\sum_{k=1}^r\phi_k\phi_k^\top$. Since
$L^s=\sum_{k=r+1}^n\lambda_k^s\phi_k\phi_k^\top$,
\[
e^{-TL^s}
=
\Pi_0+\sum_{k=r+1}^n e^{-T\lambda_k^s}\phi_k\phi_k^\top .
\]
Because $I-\Pi_0$ removes the null-space component,
\[
(I-\Pi_0)e^{-TL^s}Z_0
=
\sum_{k=r+1}^n e^{-T\lambda_k^s}\phi_k\phi_k^\top Z_0 .
\]
Using the orthonormality of the eigenvectors,
\[
\|(I-\Pi_0)e^{-TL^s}Z_0\|_F^2
=
\sum_{k=r+1}^n e^{-2T\lambda_k^s}\|\phi_k^\top Z_0\|_2^2 .
\]
Since $\lambda_k\ge\lambda_+$ for every $k$ with $\lambda_k>0$,
\[
\|(I-\Pi_0)e^{-TL^s}Z_0\|_F^2
\le
e^{-2T\lambda_+^s}
\sum_{k=r+1}^n \|\phi_k^\top Z_0\|_2^2
=
e^{-2T\lambda_+^s}\|(I-\Pi_0)Z_0\|_F^2 .
\]
Taking square roots gives the result.
\end{proof}

Proposition~\ref{prop:osi} quantifies the decay of the initial non-null
component at a rate governed by $\lambda_+^s$. A compatible source
continually injects non-null information, which leads to the sustained
response characterized below. Table~\ref{tab:fractional_order_sensitivity}
reports the index for the homogeneous term $e^{-TL^s}Y_0$, with values at
most $0.02$ at the selected times. The full state combines the decaying non-null part of this term, its
retained null-space component and the source response of
Corollary~\ref{cor:source_nonnull}.

\begin{corollary}[Source-driven non-null response]\label{cor:source_nonnull}
Assume $\Pi_0F=0$. For
\[
U(T)=e^{-TL^s}U_0+\int_0^T e^{-(T-\tau)L^s}F\,d\tau,
\]
the non-null component satisfies
\[
(I-\Pi_0)U(T)
=
e^{-TL^s}(I-\Pi_0)U_0
+
\int_0^T e^{-(T-\tau)L^s}(I-\Pi_0)F\,d\tau .
\]
Moreover,
\[
(I-\Pi_0)U(T)
\to
(L^s)^\dagger F
\quad\text{as }T\to\infty .
\]
\end{corollary}

For a fixed nonzero compatible source, $(L^s)^\dagger F\ne0$, since $F$ lies
in the orthogonal complement of $\ker(L^s)$ and $L^s$ is invertible there.
Consequently,
\[
\lim_{T\to\infty}\|(I-\Pi_0)U(T)\|_F=\|(L^s)^\dagger F\|_F>0,
\]
so when $(I-\Pi_0)U_0\neq0$ the over-smoothing index of
Definition~\ref{def:osi} applied to the full state tends to the positive
limit $\|(L^s)^\dagger F\|_F/\|(I-\Pi_0)U_0\|_F$ rather than to zero. This
gives a precise sense in which the framework mitigates oversmoothing:
source-driven propagation retains label-dependent variation outside the
null space at arbitrarily long diffusion times. Its positive-frequency
equilibrium multiplier $\lambda^{-s}$ preferentially weights low
frequencies, and the order controls the strength of that preference and
hence the balance between low-frequency structure and higher-frequency
variation in the source response.

The source-driven formulation also gives a quantitative measure of propagation speed. Consider
\[
\dot U(t)=-L^sU(t)+F,
\]
and expand $U(t)$ and $F$ in the eigenbasis $\{\phi_k\}_{k=1}^n$ of $L$:
\[
\widehat U_k(t)=\langle U(t),\phi_k\rangle,
\qquad
\widehat F_k=\langle F,\phi_k\rangle .
\]
For every nonzero eigenvalue $\lambda_k>0$, the spectral coefficient satisfies
\[
\widehat U_k(t)
=
e^{-t\lambda_k^s}\widehat U_k(0)
+
\frac{1-e^{-t\lambda_k^s}}{\lambda_k^s}\widehat F_k .
\]
Thus the homogeneous part of mode $k$ decays at rate $\lambda_k^s$, while the source contribution approaches the steady-state value
\[
\widehat U_{k,\infty}
=
\lambda_k^{-s}\widehat F_k
\]
with the same rate. The corresponding half-time is
\[
\tau_{1/2,k}(s)=\frac{\log 2}{\lambda_k^s}.
\]
For $0<s<1$, modes with $\lambda_k<1$ have shorter half-times than in the classical case $s=1$, whereas modes with $\lambda_k>1$ have longer half-times. Hence, changing $s$ redistributes the source-driven response across the spectrum rather than uniformly accelerating or slowing all modes.

When the source is compatible with the null space, $\Pi_0 F=0$, Theorem~\ref{thm:bounded} gives the long-time source response
\[
(L^s)^\dagger F
=
\sum_{\lambda_k>0}
\lambda_k^{-s}\langle F,\phi_k\rangle\phi_k .
\]
Here $(L^s)^\dagger$ denotes the Moore-Penrose pseudoinverse of $L^s$, acting only on the nonzero spectral modes. Compared with the classical Poisson response
\[
L^\dagger F
=
\sum_{\lambda_k>0}
\lambda_k^{-1}\langle F,\phi_k\rangle\phi_k ,
\]
the fractional response replaces $\lambda_k^{-1}$ by $\lambda_k^{-s}$. For two positive eigenvalues $\lambda_a<\lambda_b$, the ratio of their equilibrium weights is $(\lambda_b/\lambda_a)^{s}>1$. Thus every order preferentially weights low frequencies, and reducing $s$ flattens that preference, retaining a larger relative contribution from higher frequencies. The source sustains the non-null response, while the order tunes its spectral composition. The experiments evaluate this balance under the specified graph and supervision settings.
\subsection{Deflated Fractional Diffusion}
\label{sec:deflation}

Theorem~\ref{thm:bounded} requires $\Pi_0F=0$ for boundedness, and even then
the limit retains the initial null-space component $\Pi_0U_0$. Both effects
originate in the zero eigenvalue of $L^s$. We remove them by shifting the
entire null space, which on a disconnected graph has dimension $r>1$, so
that shifting a single null vector would not suffice. The construction is
related in motivation to the stationary-mode correction in projected label
propagation~\cite{Bozorgnia2026}, but the fixed-source properties below
follow directly from standard spectral functional
calculus~\cite{higham2008functions}.

\paragraph{Null-space convention.}
Throughout this subsection, $L\in\mathbb R^{n\times n}$ is a fixed symmetric
positive semidefinite graph Laplacian, $s>0$, and $r=\dim\ker(L)\geq1$.
Choose an orthonormal eigenbasis $\{\varphi_1,\ldots,\varphi_n\}$ such that
$\lambda_1=\cdots=\lambda_r=0$ and, when $r<n$,
$0<\lambda_{r+1}\leq\cdots\leq\lambda_n$, and set
\[
Q=[\varphi_1,\ldots,\varphi_r],
\qquad
Q^\top Q=I_r,
\qquad
\Pi_0=QQ^\top .
\]
Thus $\Pi_0$ is the orthogonal projector onto $\ker(L)=\ker(L^s)$ of
Definition~\ref{def:projector}, independent of the chosen orthonormal basis;
for a connected graph it is the rank-one projector $\varphi_1\varphi_1^\top$.
When $r<n$, $\lambda_+=\lambda_{r+1}$ is the smallest positive eigenvalue of
$L$, as in Definition~\ref{def:projector}. If $r=n$, then $L=0$ and
$\lambda_+$ is not defined; this case is treated separately below.

\begin{definition}[Full-null-space deflation]
\label{def:deflate}
For $\gamma>0$, define
\[
L^s_\gamma:=L^s+\gamma\Pi_0.
\]
\end{definition}

\begin{lemma}[Spectrum and propagator identities]
\label{lem:deflate}
The operator $L^s_\gamma$ is symmetric positive definite
and has the spectral representation
\[
L^s_\gamma
=\gamma\Pi_0+
  \sum_{k=r+1}^{n}\lambda_k^s\varphi_k\varphi_k^\top.
\]
Its least eigenvalue is
\[
\mu_\gamma:=\lambda_{\min}(L^s_\gamma)
=
\begin{cases}
\min\{\gamma,\lambda_+^s\},&r<n,\\[1mm]
\gamma,&r=n.
\end{cases}
\]
For every $t\geq0$,
\begin{align*}
e^{-tL^s_\gamma}
&=e^{-tL^s}-(1-e^{-\gamma t})\Pi_0,\\
h(tL^s_\gamma)
&=h(tL^s)+\bigl(h(\gamma t)-1\bigr)\Pi_0,
\end{align*}
where $h(x)=(1-e^{-x})/x$ for $x>0$ and $h(0)=1$.
Moreover,
\[
(L^s_\gamma)^{-1}
=(L^s)^\dagger+\gamma^{-1}\Pi_0.
\]
\end{lemma}

\begin{proof}
Since $L^s\Pi_0=\Pi_0L^s=0$, the operator
$L^s_\gamma$ acts as multiplication by $\gamma$ on
$\ker(L)$ and as multiplication by $\lambda_k^s$
on $\varphi_k$ for $k>r$. All these eigenvalues are
positive, which proves the spectral representation,
positive definiteness, and the formula for $\mu_\gamma$.
When $r=n$, the sum is empty and $L^s_\gamma=\gamma I$.

Applying the exponential to the spectral representation
replaces the unit multiplier of $e^{-tL^s}$ on the
null space by $e^{-\gamma t}$ and leaves every
positive-frequency multiplier unchanged. This gives
the first identity. The same argument with $h$, using
$h(0)=1$, gives the second identity. Finally,
\[
(L^s_\gamma)^{-1}
=\gamma^{-1}\Pi_0+
  \sum_{k=r+1}^{n}\lambda_k^{-s}
       \varphi_k\varphi_k^\top,
\]
and the sum is precisely $(L^s)^\dagger$.
\end{proof}

\begin{theorem}[Relaxation of the null-space mass]
\label{thm:deflate-mass}
Let $U_\gamma(t)\in\mathbb R^{n\times c}$ solve
\[
\dot U_\gamma=-L^s_\gamma U_\gamma+F,
\qquad U_\gamma(0)=U_0,
\]
where $F\in\mathbb R^{n\times c}$ is constant in time.
Then
\[
Q^\top U_\gamma(t)
=e^{-\gamma t}Q^\top U_0
 +\frac{1-e^{-\gamma t}}{\gamma}Q^\top F.
\]
Equivalently,
\[
\Pi_0U_\gamma(t)
=e^{-\gamma t}\Pi_0U_0
 +\frac{1-e^{-\gamma t}}{\gamma}\Pi_0F
\;\longrightarrow\;
\gamma^{-1}\Pi_0F
\quad\text{as }t\to\infty.
\]
For a connected graph, $Q=\varphi_1$ and the first
identity reduces to the scalar mass formula for each
feature column. If $\Pi_0F=0$, every null-space mass
component decays to zero at rate $\gamma$.
\end{theorem}

\begin{proof}
Set $M(t)=Q^\top U_\gamma(t)\in\mathbb R^{r\times c}$.
Since $Q^\top L^s_\gamma=\gamma Q^\top$, it satisfies
\[
\dot M(t)=-\gamma M(t)+Q^\top F,
\qquad M(0)=Q^\top U_0.
\]
Solving this linear equation entrywise gives the first
formula. Multiplication by $Q$ gives the projector
formula, and $e^{-\gamma t}\to0$ gives the limit.
\end{proof}

\begin{theorem}[Unconditional long-time limit for a fixed source]
\label{thm:deflate-limit}
Fix $L$, $s>0$, and $\gamma>0$. For every
$U_0\in\mathbb R^{n\times c}$ and every constant source
$F\in\mathbb R^{n\times c}$, the unique solution is
\[
U_\gamma(t)
=e^{-tL^s_\gamma}U_0+t\,h(tL^s_\gamma)F.
\]
It is uniformly bounded for $t\geq0$ and converges to
\[
U_{\gamma,\infty}
=(L^s_\gamma)^{-1}F
=(L^s)^\dagger F+\gamma^{-1}\Pi_0F.
\]
More precisely, with $\mu_\gamma$ as in
Lemma~\ref{lem:deflate},
\[
\|U_\gamma(t)-U_{\gamma,\infty}\|_F
\leq e^{-\mu_\gamma t}
      \|U_0-U_{\gamma,\infty}\|_F.
\]
Thus, for a fixed operator and a fixed source, the limit
is independent of the initial state. If $\Pi_0F=0$,
then
\[
U_{\gamma,\infty}=(L^s)^\dagger F,
\qquad \Pi_0U_{\gamma,\infty}=0.
\]
In contrast, the corresponding undeflated flow with the same compatible
source converges to $\Pi_0 U_0+(L^s)^\dagger F$, in agreement with
Theorem~\ref{thm:bounded}.
\end{theorem}

\begin{proof}
By Lemma~\ref{lem:deflate}, $L^s_\gamma$ is invertible.
Variation of constants and
\[
\int_0^t e^{-\tau L^s_\gamma}\,d\tau
=(L^s_\gamma)^{-1}(I-e^{-tL^s_\gamma})
=t\,h(tL^s_\gamma)
\]
give the solution formula. The inverse identity in the
lemma gives the stated expression for
$U_{\gamma,\infty}$. Subtracting this steady state yields
\[
U_\gamma(t)-U_{\gamma,\infty}
=e^{-tL^s_\gamma}(U_0-U_{\gamma,\infty}).
\]
Because $L^s_\gamma$ is symmetric positive definite,
$\|e^{-tL^s_\gamma}\|_2=e^{-\mu_\gamma t}$.
The Frobenius-norm estimate follows, proving both
uniform boundedness and convergence. If $\Pi_0F=0$,
the null-space term in the inverse formula vanishes.
The propagator identities then give, for the undeflated
solution $U(t)$ with the same $U_0$ and $F$,
\[
U(t)=U_\gamma(t)+(1-e^{-\gamma t})\Pi_0U_0.
\]
Taking $t\to\infty$ proves the final comparison.
\end{proof}
\begin{remark}[Scope of deflation]
\label{rem:deflate-scope}
Deflation shifts only the null-space eigenvalues, so every positive-frequency
multiplier, and with it Proposition~\ref{prop:rotation}, is unchanged. For
fixed $L$, $s$, $U_0$ and the same constant source $F$, Lemma~\ref{lem:deflate}
gives
\[
U_\gamma(t)-U(t)
=-(1-e^{-\gamma t})\Pi_0U_0
 +\Bigl(\tfrac{1-e^{-\gamma t}}{\gamma}-t\Bigr)\Pi_0F ,
\qquad\text{hence}\qquad
(I-\Pi_0)U_\gamma(t)=(I-\Pi_0)U(t),
\]
so the two flows have the same over-smoothing index and differ only in
$\ker(L^s)$, where a compatible source gives
$\Pi_0U_\gamma(t)=e^{-\gamma t}\Pi_0U_0$.

That difference reaches the classifier. Along the undeflated flow with
$U_0=Y_0$ and a compatible source, $q_a^\top U(t)=q_a^\top Y_0$, so for
$i\in C_a$ the null-space component $(\Pi_0U)_{i,:}=q_{a,i}(q_a^\top Y_0)$
with $q_{a,i}=\phi_i/\|\phi|_{C_a}\|_2$ is the $\phi$-weighted count of
labeled nodes of each class in the component. On $L_{\mathrm{sym}}$ this is
an additive offset toward the classes over-represented among those labeled
nodes, scaled by $\sqrt{d_i}$, which enters $\arg\max_jU_{ij}$ and which
deflation damps by $e^{-\gamma t}$; it is the degree-weighted bias addressed
by the projected iteration of~\cite{Bozorgnia2026}.

Relative to the original frequencies the homogeneous deflated multiplier is
$e^{-\gamma t}$ at $\lambda=0$ and $e^{-t\lambda^s}$ for $\lambda>0$, and for
$0<r<n$ it is nonincreasing on $\operatorname{spec}(L)$ if and only if
$0<\gamma\le\lambda_+^s$, since the null modes must not be damped more than
the slowest positive mode. The choice $\gamma=\lambda_+^s$ matches the two
relaxation timescales and is the natural reference value.
\end{remark}

Deflation is developed here as a theoretical extension; the experiments in
Section~\ref{sec:experiments} use the undeflated operator.
Lemma~\ref{lem:deflate} implements the extension through a projection
correction to the existing propagators.
\subsection{Time Discretization Schemes}
We present several time discretizations. The explicit Forward Euler scheme is
\begin{equation*}
U_{k+1} = (I - \Delta t\, L^s) U_k + \Delta t\, F.
\end{equation*}
For stability we require $\Delta t < \frac{2}{\lambda_{\max}^s}$, where $\lambda_{\max}^s$ is the largest eigenvalue of $L^s$. The implicit Backward Euler scheme
\begin{equation*}
(I + \Delta t\, L^s) U_{k+1} = U_k + \Delta t\, F
\end{equation*}
requires solving a linear system at each step but is unconditionally stable. Using the analytical solution structure, the exponential integrator
\begin{equation*}
U_{k+1} = e^{-\Delta t\, L^s} U_k + \Delta t \cdot h(\Delta t\, L^s) F_k
\end{equation*}
is exact for constant $F$ and accurate at larger steps. For high accuracy, the Runge-Kutta 4 method provides fourth-order convergence:
\begin{align*}
k_1 &= -L^s U_k + F, &
k_2 &= -L^s \big(U_k + \tfrac{\Delta t}{2} k_1\big) + F, \\
k_3 &= -L^s \big(U_k + \tfrac{\Delta t}{2} k_2\big) + F, &
k_4 &= -L^s \big(U_k + \Delta t\, k_3\big) + F, \\
U_{k+1} &= U_k + \tfrac{\Delta t}{6} (k_1 + 2k_2 + 2k_3 + k_4). &&
\end{align*}
Each scheme applies unchanged to the deflated operator $L^s_\gamma$ of Section~\ref{sec:deflation} by substituting $L^s_\gamma$ for $L^s$; the explicit stability bound then reads $\Delta t<2/\max(\gamma,\lambda_{\max}^s)$.

\subsection{Optional Self-Training Extension}
\label{subsec:optional_selftraining}
Self-training is a standard approach to augmenting supervised information
with selected pseudo-labels~\cite{Amini2024SelfTrainingSurvey}. We use it
as an optional extension of the source-driven diffusion. Starting from
$U_0=Y_0$ and the source $F_0$ of the original labeled nodes, each
iteration
\begin{enumerate}
    \item advances the propagated label-score matrix by a step $\Delta t$,
    \begin{equation}
    U_{k+1} = e^{-\Delta t\, L^s} U_k + \Delta t \cdot h(\Delta t\, L^s)\, F_k ;
    \end{equation}
    \item selects the high-confidence predictions in $U_{k+1}$;
    \item adds the selected nodes to the labeled set and recomputes the
    source from the enlarged set,
\end{enumerate}
and the iteration stops when no node is selected or after a prescribed
number of rounds. The quantities involved are defined in
Table~\ref{tab:quantsinalg}.

\begin{table}[htbp]
\centering
    \begin{tabular}{|l|l|}
    \hline
    \textbf{Symbol} & \textbf{Definition} \\ \hline\hline
$\mathcal{L}_k$ & Set of currently labeled nodes at iteration $k$ \\ \hline
$\mathcal{U}_k$ & Set of unlabeled nodes at iteration $k$, $\mathcal{U}_k = V \setminus \mathcal{L}_k$ \\ \hline
$\mathcal E_0$
& Fixed pool of nodes eligible for pseudo-label promotion
\\ \hline
$\mathcal E_k$
& Remaining eligible nodes,
  $\mathcal E_k=\mathcal E_0\setminus\mathcal L_k$
\\ \hline
$\mathcal{C}_k$ & Set of high-confidence nodes selected for pseudo-labeling at iteration $k$ \\ \hline
$Y_k$ & Current label matrix, including original and pseudo-labels \\ \hline
$U_k$ & Propagated label-score matrix at iteration $k$ \\ \hline
$\operatorname{conf}^{(s)}_i$ & Fractional confidence score for node $i$ \\ \hline
$F_k$ & Source matrix constructed from $Y_k$ and $\mathcal{L}_k$ \\ \hline
    \end{tabular}
    \caption{Quantities referred to in the algorithms of this section.}\label{tab:quantsinalg}
\end{table}

\paragraph{Eligible nodes for pseudo-label promotion.}
Let $\mathcal L_0$ be the training nodes and $\mathcal V_{\mathrm{val}}$,
$\mathcal V_{\mathrm{test}}$ the validation and test nodes, pairwise
disjoint. A fixed promotion pool
\begin{equation}
\label{eq:promotion_pool}
\mathcal E_0
\subseteq
V\setminus
\bigl(
\mathcal L_0
\cup \mathcal V_{\mathrm{val}}
\cup \mathcal V_{\mathrm{test}}
\bigr)
\end{equation}
is fixed by the protocol, by default the whole right-hand side. At round
$k$, $\mathcal L_k$ holds the training nodes and all nodes promoted so far,
and $\mathcal E_k=\mathcal E_0\setminus\mathcal L_k$ are those still
eligible; validation and test nodes take part in the propagation but are
never pseudo-labeled. The source $F_k=\mathcal C_{\phi}(Y_k,\mathcal L_k)$
is the operator of Section~\ref{sec:framework} with the components and
null vector of the diffusion Laplacian, recomputed from $Y_k$ each round,
so $\Pi_0F_k=0$ throughout.

For $i\in\mathcal E_k$ let $p_i=\operatorname{softmax}(U_{k+1}(i,:))$,
\[
a_i = \frac{\max_j p_i(j)-\tfrac{1}{c}}{1-\tfrac{1}{c}},
\qquad
b_i = 1-\frac{-\sum_{j=1}^{c}p_i(j)\log p_i(j)}{\log c},
\]
the margin above the uniform baseline and the entropy complement, both in
$[0,1]$, and
\begin{equation}\label{eq:conf}
\operatorname{conf}_i^{(s)}
=
\frac{a_i+s\,b_i}{1+s}\in[0,1].
\end{equation}
The score is a heuristic, not a calibrated probability: the softmax depends
on the scale of $U_{k+1}$ and the weight on $b_i$ on $s$, so one threshold
schedule $\theta_k$ is not equally selective across orders or datasets. The
candidates
\begin{equation}
\label{eq:eligible_candidates}
\mathcal C_k
=
\left\{
i\in\mathcal E_k:
\operatorname{conf}^{(s)}_i>\theta_k
\right\}
\end{equation}
are all promoted, $\mathcal L_{k+1}=\mathcal L_k\cup\mathcal C_k$ and
$\mathcal E_{k+1}=\mathcal E_k\setminus\mathcal C_k$, so no node is
selected twice, and the training, validation, and test nodes remain excluded.
\paragraph{Analytical check of the confidence scale.}
For $c=7$, $s=0.75$, and the hypothetical propagated score row
$U_{k+1}(i,:)=(1,0,\ldots,0)$, the stated softmax gives
$a_i\approx0.197090$ and $b_i\approx0.047422$. Thus,
$\mathrm{conf}_i^{(0.75)}\approx0.132946$, well below the initial
threshold $\theta_0=0.90$. This analytical example illustrates the sensitivity of threshold
crossing to the score scale. Source-driven evolution can produce
scores outside $[0,1]$, so the activity in a particular run depends
on its propagated score distribution.

The full procedure is Algorithm~\ref{alg:heat_kernel_self_training}.
{\scriptsize
\begin{algorithm}[H]
\caption{Fractional Heat Kernel Self-Training}
\label{alg:heat_kernel_self_training}
\footnotesize
\begin{algorithmic}[1]

\Require Symmetric graph Laplacian $L$ on $V=\{1,\ldots,n\}$ with connected
         components $C_1,\ldots,C_r$ and positive vertex degrees,
         vector $\phi$ spanning $\ker(L)$ componentwise with positive entries,
         fractional order $s>0$
\Require Nonempty initial labeled set $\mathcal L_0$,
         initial label matrix
         $Y_0\in\mathbb R^{n\times c}$ with $c\geq 2$
\Require Fixed promotion pool
         $\mathcal E_0\subseteq
         V\setminus
         (\mathcal L_0\cup\mathcal V_{\mathrm{val}}
          \cup\mathcal V_{\mathrm{test}})$
\Require Time step $\Delta t>0$,
         threshold schedule
         $\{\theta_k\}_{k=0}^{T_{\max}-1}\subseteq[0,1]$,
         positive integer iteration limit $T_{\max}$
\Ensure Latest propagated label-score matrix
        $U_{\mathrm{last}}\in\mathbb R^{n\times c}$

\State $k\leftarrow 0$
\State $U_0\leftarrow Y_0$,
       \quad $U_{\mathrm{last}}\leftarrow U_0$
\State $F_0\leftarrow
       \mathcal C_{\phi}(Y_0,\mathcal L_0)$

\While{$k<T_{\max}$}

    \State $U_{k+1}\leftarrow
        e^{-\Delta t\,L^s}U_k
        +\Delta t\,h(\Delta t\,L^s)F_k$

    \State $U_{\mathrm{last}}\leftarrow U_{k+1}$
        \Comment{Retain the new state before any early exit}

    \State $\mathcal E_k\leftarrow
        \mathcal E_0\setminus\mathcal L_k$

    \For{$i\in\mathcal E_k$}
        \State $p_i\leftarrow
            \operatorname{softmax}(U_{k+1}(i,:))$
        \State Compute $\operatorname{conf}^{(s)}_i$
            from $p_i$ using \eqref{eq:conf}
    \EndFor

    \State $\mathcal C_k\leftarrow
        \left\{
        i\in\mathcal E_k:
        \operatorname{conf}^{(s)}_i>\theta_k
        \right\}$

    \If{$\mathcal C_k=\varnothing$}
        \State \textbf{break}
    \EndIf

    \State $\mathcal L_{k+1}\leftarrow
        \mathcal L_k\cup\mathcal C_k$

    \State $Y_{k+1}\leftarrow Y_k$
        \Comment{Preserve all previously assigned labels}

    \For{$i\in\mathcal C_k$}
        \State $j_i\leftarrow
            \min\operatorname*{arg\,max}_{1\leq j\leq c}
            U_{k+1}(i,j)$
        \State $Y_{k+1}(i,:)\leftarrow e_{j_i}^{\top}$
    \EndFor

    \State $F_{k+1}\leftarrow
        \mathcal C_{\phi}(Y_{k+1},\mathcal L_{k+1})$
        \Comment{Recompute from scratch; do not accumulate}

    \State $k\leftarrow k+1$

\EndWhile

\State \Return $U_{\mathrm{last}}$

\end{algorithmic}
\end{algorithm}
}
The step $\Delta t=\infty$ is admitted in
Algorithm~\ref{alg:heat_kernel_self_training} and denotes the limit of
Theorem~\ref{thm:bounded} applied to the current state and source,
\[
U_{k+1}=\Pi_0U_k+(L^{s})^{\dagger}F_k ,
\]
which is the value taken by $e^{-\Delta tL^{s}}U_k+\Delta t\,h(\Delta tL^{s})F_k$
as $\Delta t\to\infty$ under the compatibility $\Pi_0F_k=0$ that the
componentwise source guarantees at every round. By Theorem~\ref{thm:mass},
$\Pi_0U_k=\Pi_0Y_0$ for every $k$, so the null-space component of the
iterate is the one of the initial labels throughout, and the pseudo-labels
act only through the source term.

Replacing $L^s$ by $L^s_\gamma$ gives the deflated variant. Let
\[
B_k=e^{-\Delta tL^s}U_k+\Delta t\,h(\Delta tL^s)F_k .
\]
By Lemma~\ref{lem:deflate}, the deflated update is
\[
U_{k+1}^{\gamma}
=
B_k+\Pi_0\!\left[
(e^{-\gamma\Delta t}-1)U_k
+
\Delta t\bigl(h(\gamma\Delta t)-1\bigr)F_k
\right],
\]
where $U_k$ and $F_k$ are the current state and source of the deflated
trajectory and $\Pi_0$ projects onto the full null space of the diffusion
Laplacian. One additional projection of a combined matrix therefore suffices,
with no new eigendecomposition. The projection is applied as
$\Pi_0X=Q(Q^\top X)$ and is never formed as a dense matrix: the columns of $Q$
have disjoint supports, one per connected component, so $Q^\top X$ touches each
row of $X$ once and the cost is $O(nc)$ for $c$ channels, independent of
$r=\dim\ker(L)$. If $\Pi_0F_k=0$, which holds for the componentwise source of
Section~\ref{sec:framework}, the source correction vanishes and only the state
correction remains.

{\scriptsize
\begin{algorithm}[H]
\caption{Deflated Fractional Heat Kernel Self-Training}
\label{alg:deflated}
\footnotesize
\begin{algorithmic}[1]

\Require Symmetric graph Laplacian $L$ on $V=\{1,\ldots,n\}$ with connected
         components $C_1,\ldots,C_r$ and positive vertex degrees,
         vector $\phi$ spanning $\ker(L)$ componentwise with positive entries,
         fractional order $s>0$, deflation parameter $\gamma>0$
\Require Nonempty initial labeled set $\mathcal L_0$,
         initial label matrix $Y_0\in\mathbb R^{n\times c}$ with $c\geq 2$
\Require Fixed promotion pool
         $\mathcal E_0\subseteq
         V\setminus
         (\mathcal L_0\cup\mathcal V_{\mathrm{val}}
          \cup\mathcal V_{\mathrm{test}})$
\Require Time step $\Delta t>0$,
         threshold schedule
         $\{\theta_k\}_{k=0}^{T_{\max}-1}\subseteq[0,1]$,
         positive integer iteration limit $T_{\max}$
\Ensure Latest propagated label-score matrix
        $U_{\mathrm{last}}\in\mathbb R^{n\times c}$

\State $Q\leftarrow[q_1,\ldots,q_r]$ with
       $q_a\leftarrow\phi|_{C_a}/\|\phi|_{C_a}\|_2$ extended by zero off $C_a$
       \Comment{orthonormal basis of $\ker(L^s)$; $\Pi_0=QQ^\top$}
\State Store $Q$ by its component supports and apply
       $\Pi_0X=Q(Q^\top X)$; never form $\Pi_0$
       \Comment{$O(nc)$ per application}

\State $c_E\leftarrow 1-e^{-\gamma\Delta t}$,
       \quad $c_H\leftarrow h(\gamma\Delta t)-1$
       \Comment{scalars of Lemma~\ref{lem:deflate}}
\State \textbf{Precompute} the undeflated propagators
       $e^{-\Delta t\,L^s}$ and $h(\Delta t\,L^s)$

\State $k\leftarrow 0$
\State $U_0\leftarrow Y_0$,
       \quad $U_{\mathrm{last}}\leftarrow U_0$
\State $F_0\leftarrow \mathcal C_{\phi}(Y_0,\mathcal L_0)$
       \Comment{componentwise centering, so $\Pi_0F_0=0$}

\While{$k<T_{\max}$}

    \State $B_k\leftarrow
        e^{-\Delta t\,L^s}U_k+\Delta t\,h(\Delta t\,L^s)F_k$
    \State $U_{k+1}\leftarrow
        B_k-c_E\,\Pi_0U_k+\Delta t\,c_H\,\Pi_0F_k$
        \Comment{$=e^{-\Delta tL^s_\gamma}U_k+\Delta t\,h(\Delta tL^s_\gamma)F_k$}

    \State $U_{\mathrm{last}}\leftarrow U_{k+1}$
        \Comment{Retain the new state before any early exit}

    \State $\mathcal E_k\leftarrow \mathcal E_0\setminus\mathcal L_k$

    \For{$i\in\mathcal E_k$}
        \State $p_i\leftarrow
            \operatorname{softmax}(U_{k+1}(i,:))$
        \State Compute $\operatorname{conf}^{(s)}_i$
            from $p_i$ using \eqref{eq:conf}
    \EndFor

    \State $\mathcal C_k\leftarrow
        \left\{
        i\in\mathcal E_k:
        \operatorname{conf}^{(s)}_i>\theta_k
        \right\}$

    \If{$\mathcal C_k=\varnothing$}
        \State \textbf{break}
    \EndIf

    \State $\mathcal L_{k+1}\leftarrow \mathcal L_k\cup\mathcal C_k$

    \State $Y_{k+1}\leftarrow Y_k$
        \Comment{Preserve all previously assigned labels}

    \For{$i\in\mathcal C_k$}
        \State $j_i\leftarrow
            \min\operatorname*{arg\,max}_{1\leq j\leq c}
            U_{k+1}(i,j)$
        \State $Y_{k+1}(i,:)\leftarrow e_{j_i}^{\top}$
    \EndFor

    \State $F_{k+1}\leftarrow
        \mathcal C_{\phi}(Y_{k+1},\mathcal L_{k+1})$
        \Comment{Recompute from scratch; do not accumulate}

    \State $k\leftarrow k+1$

\EndWhile

\State \Return $U_{\mathrm{last}}$

\end{algorithmic}
\end{algorithm}
}

\subsection{Graph Attention Networks with Heat Kernel Integration}\label{subsec:gat_heat}
Kipf and Welling~\cite{KipfWelling2017GCN} used the spectral properties of the graph Laplacian to design localized convolutions, which restrict propagation to nearby neighbours and can over-smooth in deep architectures~\cite{li2018deeper, oono2020graph}. Graph Attention Networks (GATs)~\cite{velickovic2018graph} introduced adaptive weighting but still aggregate locally. Heat-kernel diffusion gives them multiscale propagation.

The attention logits can be extended with heat-kernel weights,
\begin{equation}
e_{ij}^{(k)} = \text{LeakyReLU}\!\left(a^{(k)\top} [W^{(k)} h_i \,\|\, W^{(k)} h_j]\right) + \beta \cdot H_t(i,j),
\end{equation}
where $H_t(i,j) = (e^{-tL})_{ij}$, $\beta \geq 0$ controls the heat-kernel
influence and $\|$ denotes concatenation. The coefficients are
\begin{equation}
\alpha_{ij}^{(k)} = \frac{\exp(e_{ij}^{(k)})}{\sum_{l \in \mathcal{N}_i \cup \mathcal{H}_i} \exp(e_{il}^{(k)})},
\end{equation}
with $\mathcal N_i$ the direct neighbours and
$\mathcal H_i=\operatorname{TopK}_{k}(H_t(i,:))$ the $k$ nodes of largest
heat-kernel affinity to $i$, so that local structure enters through the
edges and global structure through the kernel. For $K$ heads,
\begin{equation}
h_i^{(l+1)} = \sigma\!\left(\frac{1}{K} \sum_{k=1}^K \sum_{j \in \mathcal{N}_i \cup \mathcal{H}_i} \alpha_{ij}^{(k)} W^{(k)} h_j^{(l)}\right).
\end{equation}
Phase~1 of Algorithm~\ref{alg:gat-heat-kernel} uses the plain GAT of
Section~\ref{subsec:implementation_details}, that is $\beta=0$ and
$\mathcal H_i=\emptyset$.

After training, the learned embeddings
$Z=[z_1,\ldots,z_n]^\top\in\mathbb R^{n\times d'}$ give a refined graph
$G'$. With
\[
S_{\mathrm{embed}}(i,j)=\frac{z_i^\top z_j}{\|z_i\|_2\,\|z_j\|_2}
\]
the cosine similarity of the embeddings, each node retains its $k$ most
similar others,
\begin{equation}\label{eq:candidate_pairs}
\mathcal K=\bigl\{(i,j):\,j\in\operatorname{TopK}_{k}(S_{\mathrm{embed}}(i,:)),\ j\neq i\bigr\},
\qquad k=10,
\end{equation}
with weight $(W_k)_{ij}=\max\{S_{\mathrm{embed}}(i,j),0\}$ on $\mathcal K$
and $0$ otherwise, symmetrized by the maximum
$W_k\leftarrow\max(W_k,W_k^\top)$. The refined adjacency mixes these weights
with the original ones,
\begin{equation}\label{eq:refined_weight}
A'=\tfrac12\,W_k+\tfrac12\,W ,
\end{equation}
without a threshold. A vertex may also be selected by others after
symmetrization, so its degree in $W_k$ exceeds $k$ in general, with at most
$nk$ added undirected pairs in total; evaluating $S_{\mathrm{embed}}$ only
on $\mathcal K$ keeps the construction linear in $n$ apart from the
neighbour search. Fractional diffusion is then applied on $G'$,
\begin{equation}\label{eq:refined_diffusion}
U'(t)=e^{-tL'^{s}}U_0+t\,h(tL'^{s})\,F,
\qquad
F=\mathcal C_{\phi'}(U_0,\mathcal L),
\end{equation}
where $L'$ is the symmetric normalized Laplacian of $A'+I$ and the source is
centered componentwise against $\ker(L')$ with $\phi'=D'^{1/2}\mathbf 1$, as
Remark~\ref{rem:masspractice} requires. The complete procedure is
Algorithm~\ref{alg:gat-heat-kernel}.
\begin{algorithm}[htbp]
\caption{GAT-Enhanced Heat Kernel Diffusion}
\label{alg:gat-heat-kernel}
\begin{algorithmic}[1]
\Require Graph $G=(V,W)$, labels $y$, labeled set $\mathcal L$, number of embedding neighbours $k_{\mathrm{nn}}$ ($k_{\mathrm{nn}}=10$), diffusion time $t$, fractional order $s>0$
\Require Fixed promotion pool $\mathcal E_0$ as in \eqref{eq:promotion_pool},
         threshold schedule $\{\theta_k\}_{k=0}^{T_{\max}-1}\subseteq[0,1]$,
         iteration limit $T_{\max}$
\Ensure Predicted labels $\hat y$
\State \textbf{Phase 1: GAT Feature Learning}
\State Train a multi-head GAT on $G$ using the transductive semi-supervised objective described below.
\State Extract embeddings $Z\in\mathbb R^{n\times d'}$
\State \textbf{Phase 2: Semantic Graph Construction}
\State $S_{\mathrm{embed}}(i,j)=\frac{z_i^\top z_j}{\|z_i\|_2\|z_j\|_2}$ for the $k_{\mathrm{nn}}$ nearest neighbours $j\neq i$ of each node
\State $(W_{\mathrm{nn}})_{ij}\leftarrow\max\{S_{\mathrm{embed}}(i,j),0\}$ for $j\in\operatorname{TopK}_{k_{\mathrm{nn}}}(S_{\mathrm{embed}}(i,:))$, and $0$ otherwise
\State $W_{\mathrm{nn}}\leftarrow\max(W_{\mathrm{nn}},W_{\mathrm{nn}}^\top)$
\State $A'\leftarrow\tfrac12W_{\mathrm{nn}}+\tfrac12W$; build $G'$ from $A'$, compute its connected components $C'_1,\dots,C'_{r'}$, the degree matrix $D'$ of $A'+I$, and $L'=I-D'^{-1/2}(A'+I)D'^{-1/2}$
\State \textbf{Phase 3: Heat Kernel Diffusion}
\State Set $\phi'=D'^{1/2}\mathbf 1$, the null vector of $L'$; the operator $\mathcal C_{\phi'}$ centers componentwise over $C'_1,\dots,C'_{r'}$
\State $k\leftarrow0$; initialize $Y_0$ with the labeled nodes,
       $\mathcal L_0\leftarrow\mathcal L$,
       $F_0\leftarrow\mathcal C_{\phi'}(Y_0,\mathcal L_0)$
\State $U^{(0)}(t)\leftarrow e^{-tL'^{s}}Y_0+t\cdot h(tL'^{s})F_0$
\While{$k<T_{\max}$}
    \State $\mathcal E_k\leftarrow\mathcal E_0\setminus\mathcal L_k$
    \State Compute $\operatorname{conf}^{(s)}_i$ from $U^{(k)}(t)$ by \eqref{eq:conf} for $i\in\mathcal E_k$, and set $\mathcal C_k$ by \eqref{eq:eligible_candidates}
    \If{$\mathcal C_k=\varnothing$}
        \State \textbf{break}
    \EndIf
    \State $\mathcal L_{k+1}\leftarrow\mathcal L_k\cup\mathcal C_k$,\quad
           $Y_{k+1}\leftarrow Y_k$ with $Y_{k+1}(i,:)\leftarrow e_{j_i}^{\top}$ for $i\in\mathcal C_k$
    \State $F_{k+1}\leftarrow\mathcal C_{\phi'}(Y_{k+1},\mathcal L_{k+1})$
    \State $U^{(k+1)}(t)\leftarrow e^{-tL'^{s}}U^{(k)}(t)+t\cdot h(tL'^{s})F_{k+1}$
           \Comment{advance the current state with the recomputed source}
    \State $k\leftarrow k+1$
\EndWhile
\State \Return $\hat y_i=\arg\max_j U^{(k)}_{ij}(t)$
\end{algorithmic}
\end{algorithm}
Since $W_k$ is clipped at zero and symmetrized by the pairwise maximum and
$W$ is symmetric, $A'$ is symmetric and nonnegative, so $L'$ is a symmetric
positive semidefinite Laplacian and the statements of
Section~\ref{sec:framework} apply to it with the components
$C'_1,\dots,C'_{r'}$ and the null vector $\phi'$ of $G'$ in place of those
of $G$. Symmetrization is needed: a per-row
$k$-nearest-neighbour support is not symmetric, and the spectral calculus
behind $L'^{\,s}$ and Theorems~\ref{thm:mass} and~\ref{thm:bounded} needs a
symmetric operator. Every original edge survives with weight at least
$\tfrac12W_{ij}$, so refinement can only merge components, but it need not
connect every vertex: a vertex whose cosine similarities to its selected
neighbours are all nonpositive gains no edge, since the embedding weights
are clipped at zero. The components of $G'$ are therefore recomputed in
Phase~2. Under the self-loop convention every vertex has positive degree in
$A'+I$, including a singleton component, so $\phi'$ has positive entries as
the componentwise centering of Section~\ref{sec:framework} requires.

Phase~1 trains the GAT in the standard transductive semi-supervised
setting: all node features and edges, including those of unlabeled nodes,
participate in message passing, while the supervised loss is restricted to
the labeled set,
\[
\mathcal{L}_{\mathrm{GAT}}
=
-\sum_{i\in \mathcal{L}}
\log p_\theta(y_i\mid x_i,G).
\]
The GAT is trained once, before Phase~3, and the pseudo-labels promoted in
Phase~3 do not retrain it. Phase~3 is
Algorithm~\ref{alg:heat_kernel_self_training} run on $L'$ with $\phi'$ and
step $\Delta t=t$: pseudo-labels enter through the recomputed source $F_k$,
the state is advanced from $U^{(k)}(t)$, and the loop stops when no
eligible node crosses the threshold or after $T_{\max}$ rounds. By
Theorem~\ref{thm:mass} the null-space component is
$\Pi_0U^{(k)}=\Pi_0Y_0$ at every round, so the pseudo-labels act only
through the source.
\subsection{Neural Fractional Heat Kernel Diffusion with Source}
\label{subsec:neural_fhkpp}

We extend label-space propagation to a trainable model, Neural FHK++, by applying source-driven fractional diffusion to learned hidden representations, in the spirit of GRAND++~\cite{thorpe2022grandpp}. The encoder, the source construction and the decoder are specified below.

Let $X\in\mathbb{R}^{n\times d}$ be the node-feature matrix and let $\mathcal L\subset V$ denote the set of labeled nodes. The model first maps the input features into a hidden representation
\[
Z_0=\phi_\theta(X)\in\mathbb{R}^{n\times m},
\]
where $\phi_\theta$ is a trainable encoder and $m$ is the hidden dimension. The source term is then constructed in the hidden feature space rather than in the label space. Let $\varphi$ be the vector spanning $\ker(L)$ componentwise for the Laplacian used in the diffusion. For the symmetric normalized Laplacian $L_{\mathrm{sym}}$ this vector is
\[
\varphi = D^{1/2}\mathbf 1 .
\]
The source is centered componentwise, exactly as in
Section~\ref{sec:framework}. With $C_1,\dots,C_r$ the components of the
diffusion Laplacian and $\mathcal L_a=\mathcal L\cap C_a$,
\[
\overline Z_{0,\varphi,a}
=\frac{\sum_{i\in\mathcal L_a}\varphi_i\,Z_0(i,:)}
       {\sum_{i\in\mathcal L_a}\varphi_i},
\qquad
S_\theta(i,:)
=\begin{cases}
Z_0(i,:)-\overline Z_{0,\varphi,a}, & i\in\mathcal L_a,\\[2mm]
0, & \text{otherwise},
\end{cases}
\]
so that $S_\theta=\mathcal C_\varphi(Z_0,\mathcal L)$ and, by construction,
\[
Q^\top S_\theta=0,
\qquad\text{equivalently}\qquad
\Pi_0S_\theta=0 .
\]
The hidden-space source is  compatible with the null space of the
diffusion operator, and Theorems~\ref{thm:mass} and~\ref{thm:bounded} apply to
the evolution below. 

The hidden representation is evolved according to the source-driven fractional heat equation
\[
\frac{dZ(t)}{dt}
=
-L^s Z(t)+S_\theta,
\qquad
Z(0)=Z_0,
\]
where $s>0$ is the fractional order. Since the source is time-independent during one forward pass, the exact solution at diffusion time $T>0$ is
\[
Z_T
=
e^{-T L^s}Z_0
+
\int_0^T e^{-(T-\tau)L^s}S_\theta\,d\tau .
\]
Equivalently,
\[
Z_T
=
e^{-T L^s}Z_0
+
T\,h(TL^s)S_\theta,
\qquad
h(x)=\frac{1-e^{-x}}{x},
\]
with $h(0)=1$ by continuity. The final class logits are obtained from a trainable decoder,
\[
\widehat Y=\psi_\theta(Z_T).
\]
Two optional modifications of this model are evaluated in
Section~\ref{subsec:graph_noise}. The \emph{residual} variant feeds the
decoder the sum of the diffused and the undiffused hidden states,
\[
\widehat Y=\psi_\theta(Z_T+Z_0),
\]
so that the logits retain a copy of the encoder output that the diffusion
has not smoothed. The \emph{source-strength} variant scales the source by a
factor $\kappa>0$,
\[
\frac{dZ(t)}{dt}=-L^sZ(t)+\kappa\,S_\theta,
\qquad
Z_T=e^{-TL^s}Z_0+\kappa\,T\,h(TL^s)\,S_\theta,
\]
with $\kappa$ selected on the validation set; scaling the source does not
affect its compatibility with the null space, since $\Pi_0(\kappa S_\theta)
=\kappa\,\Pi_0S_\theta=0$. We write $\kappa$ rather than
$\beta$, which denotes the heat-kernel weight in the attention logits of
Section~\ref{subsec:gat_heat}.
The encoder, decoder, and diffusion time are trained end-to-end using cross-entropy over the labeled nodes only:
\[
\mathcal L_{\mathrm{CE}}(\theta)
=
-\sum_{i\in\mathcal L}
\log
\left(
\frac{\exp(\widehat Y_{i,y_i})}
{\sum_{c=1}^{C}\exp(\widehat Y_{i,c})}
\right).
\]
Equivalently, if
\[
p_\theta(c\mid i)
=
\operatorname{softmax}(\widehat Y_i)_c,
\]
then
\[
\mathcal L_{\mathrm{CE}}(\theta)
=
-\sum_{i\in\mathcal L}\log p_\theta(y_i\mid i).
\]
 FHK++ differs from label-space heat-kernel propagation in three essential ways: the source term is constructed from learned hidden features, the diffusion acts on the hidden representation rather than directly on $Y_0$, and the classifier is trained after the diffusion step.  

\subsection{Choice of Laplacian: Original and Refined Graphs}
\label{subsec:original_refined_pipeline}

To avoid ambiguity about the graph used by the heat-kernel operator, we distinguish between two diffusion pipelines. In the original-graph pipeline, the heat-kernel operator is built from the symmetric normalized Laplacian of the input graph,
\[
L_{\rm sym}=I-D^{-1/2}AD^{-1/2},
\]
with self-loops included when specified in the experimental setup; all citation-network experiments use $A+I$ (Section~\ref{subsec:implementation_details}). The source-free heat-kernel GNN layers in Sections~\ref{subsec:heat_gcn}--\ref{subsec:multiscale_heat} use this Laplacian unless explicitly stated otherwise.

In the refined-graph pipeline the GAT baseline of
Section~\ref{subsec:implementation_details} is trained on the original
graph and its embeddings define $G'$ by Phase~2 of
Algorithm~\ref{alg:gat-heat-kernel}; the diffusion then uses
$L'_{\mathrm{sym}}$ of $A'+I$ with the components and null vector of $G'$.
Thus ``original graph'' and ``GAT-refined graph'' refer to distinct
diffusion operators in the experiments.

This distinction is important for interpreting the ablation studies. Label-space FHK on the original graph tests source-driven diffusion on the given graph structure. Label-space FHK on the GAT-refined graph tests the effect of graph refinement before diffusion. Neural FHK++ tests the hidden-space source-driven diffusion model, with or without graph refinement.
The single-weight layer of Section~\ref{subsec:heat_gcn}, the multi-scale layer
of Section~\ref{subsec:multiscale_heat}, the heat-kernel attention term of
Section~\ref{subsec:gat_heat} and the deflated operator of
Section~\ref{sec:deflation} are not evaluated in Section~\ref{sec:experiments}.

\section{Experimental Evaluation}
\label{sec:experiments}

This section evaluates the framework on a synthetic benchmark, three
citation networks and four heterophilic graphs. The Two-Moon experiment
compares source-free diffusion, compatible source-driven diffusion and
self-training. On Cora, CiteSeer and PubMed the source-driven pipeline is
compared with GAT and with its own classical-kernel version on the same
splits, and on Cora also with GraphHeat, so that the contribution of the
fractional order is separated from that of the rest of the pipeline. Further experiments cover the sensitivity to the order, heterophilic
graphs, robustness to edge rewiring, and the position of the results
relative to published baselines.

\subsection{Two-Moon Dataset Evaluation}
\label{subsec:twomoon}

The Two-Moon dataset has $1000$ points with noise level $0.15$ and two known
classes, which makes it a controlled setting for comparing three source
variants across fractional orders and label budgets:
\begin{itemize}
\item Scheme~1, source-free diffusion, $U(t)=e^{-tL^s}U_0$;
\item Scheme~2, diffusion with a fixed compatible source,
$U(t)=e^{-tL^s}U_0+t\,h(tL^s)F$ with $F=\mathcal C_{\phi}(U_0,\mathcal L)$;
\item Scheme~3, diffusion with self-training, where the source
$F_k=\mathcal C_{\phi}(Y_k,\mathcal L_k)$ is recomputed after each round of
pseudo-label promotion.
\end{itemize}
Each of $50$ trials regenerates the point set and draws a balanced labeled
set, $100$ validation nodes, whose labels are used to select the diffusion
time, and $400$ test nodes. The graph is a full
Gaussian graph,
\[
W_{ij}=\exp(-\|x_i-x_j\|^2/2\sigma^2),
\]
with unit
self-loops, where $\sigma$ is $0.05$ times the median pairwise distance,
fixed by a preliminary sweep over $\{0.05,0.07,0.10,0.15,0.20,0.30\}$ on
$25$ trials whose seeds are disjoint from the reported ones; the operator is
$L_{\mathrm{sym}}$ with $\phi=\sqrt d$, and the graph was connected in every
trial. For every scheme, order and trial the diffusion time is selected from
$t\in\{0.5,1,2,4,8,16,32,64\}$ by validation accuracy; for Schemes~2 and~3
the grid also contains the long-time limit $t=\infty$, which is
$\Pi_0Y_0+(L^s)^\dagger F$ by Theorem~\ref{thm:bounded} and is used in every
round of Scheme~3. Scheme~3 uses $\theta_0=0.7$, increased by $0.02$ per
round, at most $10$ rounds and the promotion pool
\eqref{eq:promotion_pool}. The three schemes share each trial's point set
and label draw, so their differences are paired. Figure~\ref{fig:1} shows
the trial of median test accuracy for Scheme~2 at $s=1$ with one label per
class.

\begin{figure}[t]
    \centering
    \includegraphics[width=\linewidth]{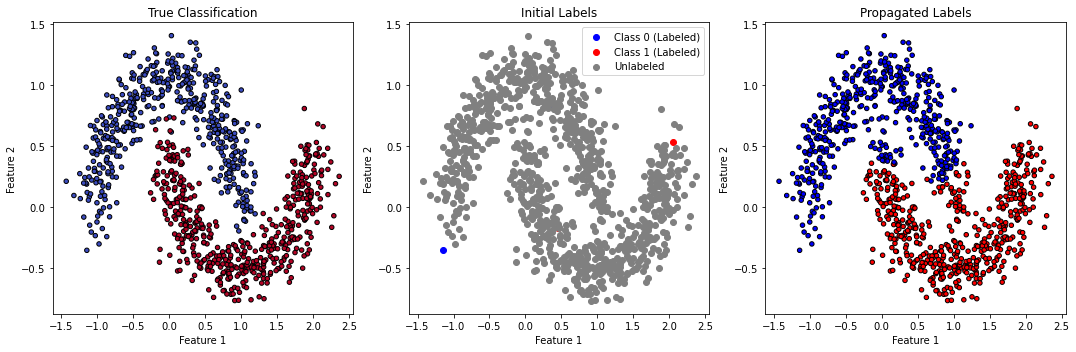}
    \caption{Two-Moon with one labeled point per class: true classes,
    initial labels, and propagated labels for the trial of median test
    accuracy among the $50$ trials of Table~\ref{tab:fractional_schemes}
    for Scheme~2 at $s=1$. The validation-selected time is the long-time
    limit $t=\infty$; the test accuracy is $0.980$, and $23$ of the $1000$
    points are misclassified (circled), all in the overlap of the two
    moons.}
    \label{fig:1}
\end{figure}

\begin{table}[htbp]
\centering
\caption{Two-Moon: mean accuracy on the $400$ test nodes $\pm$ standard
error over $50$ trials for the three diffusion schemes. Within a trial the
schemes share the point set, the label draw and the validation and test
nodes. Schemes~2 and~3 coincide at $s=0.2$, where no node crosses the
promotion threshold in any trial.}
\label{tab:fractional_schemes}
\begin{tabular}{ccccc}
\toprule
\textbf{\(s\)} & \textbf{Labels per class} & \textbf{Scheme 1} & \textbf{Scheme 2} & \textbf{Scheme 3} \\
\midrule
0.2 & 1 & 0.886 $\pm$ 0.016 & 0.928 $\pm$ 0.014 & 0.928 $\pm$ 0.014 \\
0.2 & 2 & 0.938 $\pm$ 0.009 & 0.966 $\pm$ 0.004 & 0.966 $\pm$ 0.004 \\
0.2 & 3 & 0.960 $\pm$ 0.008 & 0.971 $\pm$ 0.004 & 0.971 $\pm$ 0.004 \\
0.2 & 4 & 0.965 $\pm$ 0.006 & 0.977 $\pm$ 0.003 & 0.977 $\pm$ 0.003 \\
0.2 & 5 & 0.978 $\pm$ 0.002 & 0.979 $\pm$ 0.002 & 0.979 $\pm$ 0.002 \\
0.2 & 6 & 0.982 $\pm$ 0.003 & 0.982 $\pm$ 0.002 & 0.982 $\pm$ 0.002 \\
\midrule
0.8 & 1 & 0.894 $\pm$ 0.014 & 0.962 $\pm$ 0.006 & 0.969 $\pm$ 0.006 \\
0.8 & 2 & 0.940 $\pm$ 0.007 & 0.976 $\pm$ 0.003 & 0.982 $\pm$ 0.002 \\
0.8 & 3 & 0.954 $\pm$ 0.007 & 0.980 $\pm$ 0.002 & 0.980 $\pm$ 0.002 \\
0.8 & 4 & 0.967 $\pm$ 0.005 & 0.983 $\pm$ 0.001 & 0.984 $\pm$ 0.001 \\
0.8 & 5 & 0.976 $\pm$ 0.002 & 0.980 $\pm$ 0.002 & 0.981 $\pm$ 0.002 \\
0.8 & 6 & 0.978 $\pm$ 0.004 & 0.984 $\pm$ 0.002 & 0.986 $\pm$ 0.001 \\
\midrule
1.0 & 1 & 0.817 $\pm$ 0.017 & 0.971 $\pm$ 0.005 & 0.964 $\pm$ 0.007 \\
1.0 & 2 & 0.886 $\pm$ 0.012 & 0.978 $\pm$ 0.002 & 0.979 $\pm$ 0.003 \\
1.0 & 3 & 0.930 $\pm$ 0.010 & 0.980 $\pm$ 0.002 & 0.979 $\pm$ 0.003 \\
1.0 & 4 & 0.949 $\pm$ 0.008 & 0.982 $\pm$ 0.001 & 0.983 $\pm$ 0.001 \\
1.0 & 5 & 0.963 $\pm$ 0.004 & 0.977 $\pm$ 0.002 & 0.979 $\pm$ 0.002 \\
1.0 & 6 & 0.970 $\pm$ 0.006 & 0.981 $\pm$ 0.002 & 0.984 $\pm$ 0.001 \\
\bottomrule
\end{tabular}
\end{table}

The paired source comparison shows the benefit of sustained supervision at
sparse label budgets. Scheme~2 improves on Scheme~1 in all eighteen
configurations; the pointwise $95\%$ paired $t$-interval excludes zero in
fourteen, without adjustment across configurations. At $s=1$, the gains
are $0.154\pm0.035$ in accuracy at one label per class and
$0.092\pm0.023$ at two. Source-driven mean accuracy exceeds $0.96$ from
one label per class at $s=0.8$ and $s=1$, and from two labels per class
at every tested order; the residual error at one label comes from trials
in which a labeled point falls in the overlap of the two moons, as in
Figure~\ref{fig:1}. At $s=1$, validation selects $t=\infty$ in every
trial at one label per class, yielding mean test accuracy $0.971$. This
demonstrates useful classification at the sustained source equilibrium
characterized by Theorem~\ref{thm:bounded}.

The fractional order also improves propagation independently of the
source. Within Scheme~1, $s=0.2$ exceeds $s=1$ by $0.012$ to $0.069$ in accuracy
across the six budgets, with all six pointwise paired intervals excluding
zero. The diffusion time is selected separately for each order, so this
gain compares the orders under their validation-selected times
(Remark~\ref{rem:unequal_diffusion_times}). Within Scheme~2, $s=1$
exceeds $s=0.2$ by $0.044\pm0.023$ at one label per class and
$0.012\pm0.006$ at two; the differences are unresolved from three labels
on. The selected horizon also varies: at $s=1$, the share choosing
$t=\infty$ falls from $100\%$ at one label per class to $38\%$ at six,
while $s=0.2$ selects finite times throughout. The source and order thus
provide distinct design choices, with the source-driven equilibrium
weighted by $\lambda_k^{-s}$ and the useful horizon selected from data.

Self-training promotes no node at $s=0.2$ in any trial, so Scheme~3
coincides with Scheme~2 there. At $s=0.8$ and $s=1$ the long-time limit
produces scores large enough for \eqref{eq:conf} to cross $\theta_0=0.7$,
and nodes are promoted in $17$ to $44$ of the $50$ trials, but the paired
difference between Schemes~3 and~2 lies between $-0.007$ and $+0.007$ and
its interval excludes zero only at $s=0.8$ with one, two and six labels per
class ($+0.007\pm0.005$, $+0.005\pm0.004$ and $+0.002\pm0.002$). The
corresponding measurement on Cora is reported in
Section~\ref{subsec:fractional_order_sensitivity}.

\subsection{Citation Network Datasets}
\label{sec:citation}

Cora, CiteSeer and PubMed are citation networks whose nodes are papers and
whose edges are citations. Cora has $2{,}708$ papers in seven research areas
and $5{,}278$ undirected citation links, with $A_{ij}=1$ if either paper
cites the other. The original-graph pipeline diffuses on this graph; the
GAT-refined pipeline of Section~\ref{subsec:original_refined_pipeline}
starts from the same graph and adds embedding-induced edges by Phase~2 of
Algorithm~\ref{alg:gat-heat-kernel}. The operator is the symmetric
normalized Laplacian of $\widetilde A=A+I$.

\subsubsection{Implementation Details and Hyperparameter Selection}
\label{subsec:implementation_details}

All citation-network experiments use
$L_{\mathrm{sym}}=I-\widetilde D^{-1/2}\widetilde A\widetilde D^{-1/2}$ with
$\widetilde A=A+I$ and $\widetilde D_{ii}=\sum_j\widetilde A_{ij}$. The
source of every diffusion step is the componentwise operator
$\mathcal C_\phi$ of Section~\ref{sec:framework} with
$\phi=\widetilde D^{1/2}\mathbf 1$, so that $\Pi_0F=0$ on the disconnected
citation graphs and Theorem~\ref{thm:bounded} applies; on the GAT-refined
graph $G'$ the components and the null vector are those of $G'$.

Each of $10$ random splits samples $k\in\{1,2,3,4,5,10,20\}$ labeled nodes
per class for training, with $500$ validation and $1000$ test nodes disjoint
from them; $k$ counts only the training nodes that initialize the source and
train the neural components. Throughout, ``$k$ labels per class''
refers to the training labels: the labels of the $500$ validation nodes are
used for model selection and early stopping, so the protocol is not one of
learning from $k$ labeled nodes per class alone.
Table~\ref{tab:implementation_hyperparameters} lists every setting and
whether it was fixed in advance or selected on the validation nodes. The
diffusion time, the early-stopping epoch of the GAT and, where stated, the
order $s$ are selected on validation data; the test set is used once, for
the final accuracy.

For Cora and CiteSeer the propagators are evaluated exactly in the
eigenbasis, and $t=\infty$ in the time grid denotes the long-time limit
$\Pi_0Y_0+(L^s)^\dagger F$ of Theorem~\ref{thm:bounded}. For PubMed the
propagators $e^{-tL^s}$ and $t\,h(tL^s)$ are applied by Chebyshev expansions
with Gauss--Chebyshev coefficients on $[0,\lambda_{\max}]$, and the grid
contains no $t=\infty$. Self-training uses the promotion pool
\eqref{eq:promotion_pool}. The GAT loss is computed on the labeled nodes,
with all features and edges used in message passing, and the refined graph
is built from its embeddings by Phase~2 of
Algorithm~\ref{alg:gat-heat-kernel}. The GraphHeat arms of
Table~\ref{tab:cora_main} use the architecture and matched sparsity budget
of Section~\ref{subsec:graphheat_controlled} on the same splits as the
other methods. Neural FHK++ (Tables~\ref{tab:rewiring_full_ablation}
and~\ref{tab:literature_comparison}) trains the encoder, the decoder and
the diffusion time end-to-end. All intervals in this section are $95\%$
$t$-intervals over the splits unless a caption states otherwise; matched
differences and their intervals are computed per split.

\begin{table}[t]
\centering
\small
\caption{Settings of the citation-network experiments
(Tables~\ref{tab:cora_main}, \ref{tab:datasets}
and~\ref{tab:fractional_order_sensitivity}). ``Selected'' means chosen on
the validation nodes of each split; ``fixed'' means set in advance. Two-Moon
settings are given in Section~\ref{subsec:twomoon}. The Neural FHK++ rows apply to
Tables~\ref{tab:rewiring_full_ablation} and~\ref{tab:literature_comparison}.}
\label{tab:implementation_hyperparameters}
\setlength{\tabcolsep}{3.5pt}
\renewcommand{\arraystretch}{1.15}
\begin{tabularx}{\linewidth}{
>{\raggedright\arraybackslash}p{0.30\linewidth}
>{\raggedright\arraybackslash}X
>{\raggedright\arraybackslash}p{0.22\linewidth}}
\toprule
Setting & Values & Status \\
\midrule
Random splits & $10$; $k$ labels per class, $500$ validation, $1000$ test nodes & fixed \\
Laplacian & $L_{\mathrm{sym}}$ of $A+I$; componentwise source $\mathcal C_\phi$, $\phi=\widetilde D^{1/2}\mathbf 1$ & fixed \\
Fractional order $s$ & $0.75$ (Tables~\ref{tab:cora_main}, \ref{tab:datasets}); grid $\{0.30,0.50,0.75,1.00,1.25\}$ (Table~\ref{tab:fractional_order_sensitivity}, Section~\ref{subsec:graphheat_controlled}) & fixed / selected \\
Diffusion time $t$ & $\{0.5,1,2,4,8,16,32\}\cup\{\infty\}$ with the exact eigendecomposition (Cora, CiteSeer); $\{0.5,1,2,4,8,16,32\}$ with the Chebyshev expansion (PubMed) & selected per split and order \\
Chebyshev degree $M$ & $64$; relative error below $2\times10^{-3}$ at $s=0.75$, $t\le32$ & fixed \\
Refined graph & cosine $10$-NN of the GAT embeddings, clipped at $0$, symmetrized by the maximum; $A'=\tfrac12 W_k+\tfrac12 A$ & fixed \\
Pseudo-label threshold $\theta_0$ & $0.90$; $\theta_k=\min\{0.99,\theta_0+0.01k\}$ & fixed \\
Maximum self-training rounds & $5$ & fixed \\
GAT architecture & $2$ layers, $8$ heads $\times$ $8$ hidden units, dropout $0.6$ & fixed \\
GAT optimization & Adam, learning rate $0.005$, weight decay $5\times10^{-4}$, at most $1000$ epochs, patience $100$ on validation accuracy & fixed \\
GraphHeat arms & $2$ layers, $16$ hidden units, dropout $0.5$, Adam at $0.01$, $t=3.5$, $269$ stored entries per node; exponent $1$ or selected from $\{0.30,0.50,0.75,1.00,1.25\}$ & fixed / selected \\
Neural FHK++ hidden dimension & $\{32,64,128\}$ & selected \\
Neural FHK++ dropout & $\{0.3,0.5\}$ & selected \\
Neural FHK++ learning rate & $\{0.005,0.01\}$; weight decay $5\times10^{-4}$ & selected / fixed \\
Neural FHK++ training & $T=\operatorname{softplus}(\widetilde T)$ learned end-to-end; at most $1000$ epochs, patience $100$ & fixed \\
\bottomrule
\end{tabularx}
\end{table}

Table~\ref{tab:cora_main} compares the methods on Cora under one protocol.
The paired $t$-tests of the fractional pipeline against GAT form one family
of seven comparisons and are adjusted by the Holm--Bonferroni
procedure~\cite{holm1979}.

\begin{table}[t]
\centering
\footnotesize
\setlength{\tabcolsep}{3.5pt}
\caption{Cora, the same $10$ splits throughout: GAT, GraphHeat with matched
sparsity at the classical and at the validation-tuned exponent, and
label-space diffusion on the GAT-refined graph at the classical ($s{=}1$)
and the fractional ($s{=}0.75$) kernel; mean $\pm$ $95\%$ $t$-interval.
Holm $p$ is the Holm-adjusted paired $t$-test of the fractional pipeline
against GAT across the seven budgets, and $\Delta$ the matched per-split
difference between the two kernels of the pipeline. The GraphHeat arms were
run at one to five labels per class.}
\label{tab:cora_main}
\begin{tabular}{cccccccc}
\toprule
Labels & & \multicolumn{2}{c}{GraphHeat} & \multicolumn{2}{c}{Pipeline} & & \\
\cmidrule(lr){3-4}\cmidrule(lr){5-6}
per class & GAT & $s{=}1$ & tuned $s$ & $s{=}1$ & $s{=}0.75$ & Holm $p$ & $\Delta$ vs.\ $s{=}1$ \\
\midrule
1  & $48.1 \pm 5.4$ & $54.3 \pm 6.6$ & $55.0 \pm 6.6$ & $57.5 \pm 6.6$ & $57.3 \pm 6.3$ & $0.001$ & $-0.3 \pm 0.7$ \\
2  & $57.2 \pm 7.8$ & $61.4 \pm 5.9$ & $61.8 \pm 5.5$ & $62.5 \pm 5.4$ & $62.2 \pm 5.5$ & $0.134$ & $-0.3 \pm 0.5$ \\
3  & $65.4 \pm 4.8$ & $68.7 \pm 3.5$ & $68.4 \pm 3.4$ & $67.9 \pm 4.3$ & $67.9 \pm 4.3$ & $0.097$ & $-0.0 \pm 0.2$ \\
4  & $69.5 \pm 4.2$ & $71.4 \pm 3.1$ & $71.0 \pm 3.2$ & $72.0 \pm 2.9$ & $71.8 \pm 3.0$ & $0.097$ & $-0.2 \pm 0.2$ \\
5  & $70.9 \pm 3.3$ & $72.4 \pm 2.9$ & $72.5 \pm 2.9$ & $72.3 \pm 3.0$ & $72.2 \pm 3.0$ & $0.134$ & $-0.1 \pm 0.5$ \\
10 & $76.4 \pm 1.2$ & ---            & ---            & $76.0 \pm 1.3$ & $76.0 \pm 1.3$ & $0.220$ & $+0.0 \pm 0.2$ \\
20 & $80.2 \pm 0.6$ & ---            & ---            & $79.2 \pm 0.7$ & $79.1 \pm 0.7$ & $0.118$ & $-0.1 \pm 0.2$ \\
\bottomrule
\end{tabular}
\end{table}

Table~\ref{tab:cora_main} shows the largest improvements over GAT at the
sparsest training budgets. At one label per class, classical GraphHeat
improves mean accuracy by $6.2$ percentage points and the fractional
pipeline by $9.2$ points, the latter difference surviving Holm adjustment;
at two to five labels per class the pipeline's mean gains range from $1.3$
to $5.0$ points, and at ten and twenty labels GAT has the higher mean by
$0.4$ and $1.1$ points.

The matched classical and fractional pipeline results differ by at most
$0.3$ points. At $s=1$ the pipeline exceeds GraphHeat by $3.2$ points at
one label per class and $1.1$ at two; these are complete-method comparisons
that include the effects of representation learning, graph refinement,
label-source propagation and the stated time and sparsification policies,
and the only paired comparison of the source itself is that of Schemes~1
and~2 on Two-Moon (Section~\ref{subsec:twomoon}). Within GraphHeat on these
ten splits, validation-based exponent selection yields a paired gain of
$0.7\pm0.5$ points at one label per class and no consistent difference at
larger budgets. Table~\ref{tab:datasets} extends the pipeline and GAT
comparison to CiteSeer and PubMed.

\begin{table}[t]
\centering
\small
\caption{CiteSeer and PubMed, one to five labels per class: GAT baseline and
label-space diffusion on the GAT-refined graph with the classical and the
fractional kernel, mean $\pm$ $95\%$ $t$-interval over $10$ splits, and the
matched per-split difference between the kernels. PubMed propagators are
applied by a Chebyshev expansion of degree $64$.}
\label{tab:datasets}
\begin{tabular}{lccccc}
\toprule
Dataset & Labels/class & GAT & Classical ($s{=}1$) & Fractional ($s{=}0.75$) & $\Delta$ vs.\ $s{=}1$ \\
\midrule
CiteSeer & 1 & $37.6 \pm 6.5$ & $45.7 \pm 6.5$ & $45.6 \pm 6.7$ & $-0.1 \pm 0.5$ \\
CiteSeer & 2 & $41.8 \pm 5.7$ & $48.6 \pm 6.7$ & $48.4 \pm 6.7$ & $-0.1 \pm 0.4$ \\
CiteSeer & 3 & $49.7 \pm 5.5$ & $54.8 \pm 3.9$ & $54.8 \pm 4.0$ & $-0.0 \pm 0.3$ \\
CiteSeer & 4 & $53.4 \pm 2.7$ & $58.7 \pm 3.4$ & $58.9 \pm 3.3$ & $+0.3 \pm 0.4$ \\
CiteSeer & 5 & $58.9 \pm 3.2$ & $62.1 \pm 2.3$ & $61.8 \pm 2.3$ & $-0.3 \pm 0.2$ \\
\midrule
PubMed & 1 & $56.4 \pm 5.3$ & $56.1 \pm 4.5$ & $56.3 \pm 4.6$ & $+0.1 \pm 0.2$ \\
PubMed & 2 & $59.4 \pm 4.4$ & $60.4 \pm 3.3$ & $61.1 \pm 3.6$ & $+0.7 \pm 1.1$ \\
PubMed & 3 & $63.3 \pm 3.4$ & $63.6 \pm 2.8$ & $64.2 \pm 2.8$ & $+0.6 \pm 0.9$ \\
PubMed & 4 & $70.8 \pm 2.2$ & $70.1 \pm 2.3$ & $70.3 \pm 2.2$ & $+0.1 \pm 0.7$ \\
PubMed & 5 & $70.6 \pm 1.7$ & $70.5 \pm 1.1$ & $70.1 \pm 1.3$ & $-0.3 \pm 0.4$ \\
\bottomrule
\end{tabular}
\end{table}

On CiteSeer the fractional pipeline improves mean accuracy over GAT by
$8.0$, $6.6$, $5.1$, $5.5$ and $2.9$ percentage points from one to five
training labels per class, and the classical pipeline gives similar gains;
on PubMed the three methods are within $1.7$ points of one another at each
budget. The Holm-adjusted tests against GAT were run for Cora only
(Table~\ref{tab:cora_main}), so the CiteSeer margins rest on the intervals
over splits shown in Table~\ref{tab:datasets}. The paired kernel
differences have magnitude at most $0.7$ points on both graphs; at five
labels per class on CiteSeer the classical kernel is ahead by $0.3\pm0.2$
points. The following diagnostic examines the spectral distribution of the
class information across the datasets.

\subsubsection{Post-hoc label-spectral diagnostic}
\label{subsec:label_spectral_alignment}

Let $\mathsf Y\in\{0,1\}^{n\times c}$ be the complete one-hot class-label
matrix, distinct from the masked training-label matrix $Y_0$. For an
orthonormal eigenbasis of the diffusion Laplacian,
$L_{\mathrm{sym}}\phi_k=\lambda_k\phi_k$, the spectral label
energy~\cite{chung1997spectral} is
\[
E_k=\|\phi_k^\top\mathsf Y\|_2^2,
\qquad
\sum_k E_k=\|\mathsf Y\|_F^2=n_{\mathcal Y},
\]
where $n_{\mathcal Y}$ is the number of labeled nodes ($n_{\mathcal Y}=n$
except on CiteSeer, whose Planetoid version has $15$ unlabeled isolated
vertices). The shares relative to the equal-time crossing are
\[
\rho_{<1}=\frac{\sum_{\lambda_k<1}E_k}{n_{\mathcal Y}},
\qquad
\rho_{=1}=\frac{\sum_{\lambda_k=1}E_k}{n_{\mathcal Y}},
\qquad
\rho_{>1}=\frac{\sum_{\lambda_k>1}E_k}{n_{\mathcal Y}},
\]
and, with $\Pi_0$ the orthogonal projector onto the null space of
$L_{\mathrm{sym}}$,
\[
\rho_0=\frac{\|\Pi_0\mathsf Y\|_F^2}{n_{\mathcal Y}},
\qquad
\rho_{(0,1)}=\rho_{<1}-\rho_0.
\]
Only $\rho_{(0,1)}$ is subject to the stronger damping of
Proposition~\ref{prop:rotation}; $\rho_0$ is unchanged by $s$. Cora and
CiteSeer use the exact eigendecomposition. For PubMed, polynomial spectral
filtering gives approximate energy shares; its entry in the $\rho_{=1}$
column estimates the energy in a transition band of half-width $10^{-3}$
around one. The complete
labels enter only this post-hoc diagnostic, never the choice of order,
time, graph or thresholds. Values are reported in
Table~\ref{tab:label_spectral}.

\subsection{Controlled Comparison with GraphHeat}
\label{subsec:graphheat_controlled}

Table~\ref{tab:cora_main} compares complete methods. This subsection
changes nothing but the power of the Laplacian inside
GraphHeat~\cite{xu2020graph}, the closest prior method on the operator
axis, so that the exponent is isolated in a host that has trained weights
and propagates node features but has no source; the differences from the
pipeline are listed after Table~\ref{tab:cora_main}.

\paragraph{Reproduction.}
GraphHeat is reimplemented as published: two layers with $16$ hidden units,
the identity-plus-kernel layer of Section~\ref{sec:gnn}, dropout $0.5$,
Glorot initialization, Adam at learning rate $0.01$, cross-entropy on the
labeled nodes, termination after $200$ epochs without improvement of the
validation loss, and propagation matrix $\mathcal T_\epsilon(e^{-tL})$ with
$L$ built without self-loops. Table~\ref{tab:graphheat_reproduction} checks
the reproduction on the public splits ($20$ labels per class, $500$
validation and $1000$ test nodes) with the published $t$ and $\epsilon$.

\begin{table}[t]
\centering
\small
\caption{Reproduction of GraphHeat on the standard public split with the
diffusion time $t$ and threshold $\epsilon$ of~\cite{xu2020graph}; mean
$\pm$ 95\% CI over 10 initialization seeds, on the two graphs for which the
exact kernel fits in memory.}
\label{tab:graphheat_reproduction}
\begin{tabular}{lccccc}
\toprule
Dataset & $t$ & $\epsilon$ & Published & Our reproduction & Difference \\
\midrule
Cora     & $3.5$ & $10^{-4}$ & $83.7$ & $83.58 \pm 0.16$ & $-0.12$ \\
CiteSeer & $4.5$ & $10^{-5}$ & $72.5$ & $72.81 \pm 0.29$ & $+0.31$ \\
\bottomrule
\end{tabular}
\end{table}

\paragraph{Matched sparsity.}
The entries of $e^{-tL^{s}}$ decay more slowly in the graph metric as $s$
decreases, so a fixed threshold $\epsilon$ retains a much larger
neighbourhood for the fractional kernel (Table~\ref{tab:graphheat_density}):
at $s=0.5$ more than five times as many entries as at $s=1$, and a
propagation matrix more than half dense. A comparison at fixed $\epsilon$
would confound the exponent with the receptive field and the cost. The
threshold is replaced by a common budget of $269$ stored entries per node,
the mean row count of GraphHeat's own rule at $s=1$: the diagonal is kept,
the remaining budget goes to the largest off-diagonal entries chosen
globally as symmetric pairs, and every order stores $728452$ entries.
Because this support differs from the original rule even at $s=1$, the arm
is called GraphHeat with matched sparsity.

\begin{table}[t]
\centering
\small
\caption{Sparsity of $\mathcal T_\epsilon(e^{-tL^{s}})$ on Cora at
$t=3.5$ and $\epsilon=10^{-4}$, the tuned GraphHeat setting, with
$L=I-D^{-1/2}AD^{-1/2}$ formed without self-loops.}
\label{tab:graphheat_density}
\begin{tabular}{ccccc}
\toprule
$s$ & Retained fraction of $n^2$ & Mean neighbours & Median & Maximum \\
\midrule
$1.00$ & $0.099$ & $269$  & $254$  & $1353$ \\
$0.75$ & $0.216$ & $586$  & $575$  & $2333$ \\
$0.50$ & $0.544$ & $1472$ & $1656$ & $2485$ \\
$0.30$ & $0.781$ & $2114$ & $2446$ & $2485$ \\
\bottomrule
\end{tabular}
\end{table}

\paragraph{Classical versus fractional exponent.}
Table~\ref{tab:graphheat_m1} changes the exponent alone; architecture,
Laplacian convention, diffusion time $t=3.5$, sparsity budget, splits and
initialization seeds are shared. The tuned arm selects $s$ on validation
data from $\{0.30,0.50,0.75,1.00,1.25\}$, a grid that contains the classical
exponent, so GraphHeat is a special case of it. Differences are matched per
split with Holm adjustment across the five budgets, over $30$ splits because
the effect is smaller than the split-to-split variation. The gain is
confined to the sparsest supervision: $1.45$ points at one label per class,
surviving the correction, and within noise of zero from three labels on.
Its size depends on the split set: on the ten splits of
Table~\ref{tab:cora_main} the tuned arm is worth $+0.7\pm0.5$ points at one
label per class, and the classical arm differs by five points in absolute
accuracy between the two sets ($54.3$ against $49.13$).

\begin{table}[t]
\centering
\small
\caption{Operator isolation on Cora at the common diffusion time
$t=3.5$ and a matched sparsity budget of $269$ stored entries per node,
$30$ random splits. The tuned arm selects the exponent on validation
data; $\Delta$ is the matched per-split difference against the classical
exponent.}
\label{tab:graphheat_m1}
\begin{tabular}{ccccc}
\toprule
Labels/class & GraphHeat ($s{=}1$) & Tuned exponent & $\Delta$ (95\% CI) & Holm $p$ \\
\midrule
1 & $49.13$ & $50.58$ & $+1.45 \pm 0.84$ & $0.007$ \\
2 & $60.42$ & $61.07$ & $+0.65 \pm 0.63$ & $0.174$ \\
3 & $68.04$ & $68.32$ & $+0.28 \pm 0.61$ & $1.000$ \\
4 & $69.41$ & $69.49$ & $+0.08 \pm 0.23$ & $1.000$ \\
5 & $73.76$ & $73.87$ & $+0.11 \pm 0.30$ & $1.000$ \\
\bottomrule
\end{tabular}
\end{table}

Tables~\ref{tab:graphheat_m1} and~\ref{tab:graphheat_fixed} distinguish
adaptive exponent selection from a fixed-order policy. Validation-based
selection gives the positive paired gain of Table~\ref{tab:graphheat_m1} at
one label per class under the common-time, matched-storage protocol, and
selects an order below one on $9$ of $30$ splits at that budget and on $3$
of $30$ at five labels per class. For fixed $s=0.75$ the mean differences
from the classical arm range from $-0.56$ to $-1.97$ percentage points;
smaller orders have larger losses under this protocol, while $s=1.25$ is
within a point of the classical kernel from two labels on. These results
support validation-based exponent selection within the tested family and
show its dependence on the stated time and sparsification policies.

\begin{table}[t]
\centering
\small
\caption{Accuracy of GraphHeat with matched sparsity at each fixed exponent
$s$ minus its accuracy at $s=1$ (the classical column of
Table~\ref{tab:graphheat_m1}), matched per split; Cora, common diffusion
time $t=3.5$, $30$ splits. Negative entries are losses relative to the
classical kernel.}
\label{tab:graphheat_fixed}
\begin{tabular}{cccccc}
\toprule
$s$ & $k{=}1$ & $k{=}2$ & $k{=}3$ & $k{=}4$ & $k{=}5$ \\
\midrule
$0.30$ & $-23.04$ & $-26.23$ & $-26.39$ & $-27.63$ & $-22.97$ \\
$0.50$ & $-6.85$  & $-10.40$ & $-9.71$  & $-10.30$ & $-8.57$ \\
$0.75$ & $-0.56$  & $-1.97$  & $-1.26$  & $-1.66$  & $-1.62$ \\
$1.00$ & $0$      & $0$      & $0$      & $0$      & $0$ \\
$1.25$ & $-2.23$  & $-0.97$  & $-0.54$  & $-0.41$  & $+0.05$ \\
\bottomrule
\end{tabular}
\end{table}

\paragraph{Truncation.}
The common storage budget gives each order the same number of propagation
entries. At $s=0.30$, $2114$ entries per node exceed $10^{-4}$ before
sparsification, against the retained budget of $269$
(Table~\ref{tab:graphheat_density}), so the amount of discarded kernel mass
depends on the order, and Table~\ref{tab:graphheat_fixed} evaluates the
combined filter and sparsification policy. The spectral formulas describe
the full kernel; an unsparsified comparison would assess the corresponding
operators directly.

\paragraph{Time matching.}
The common-time experiment uses the classical GraphHeat setting $t=3.5$, at
which an order below one damps the band $(0,1)$, which carries most of
Cora's label energy (Table~\ref{tab:label_spectral}), much more strongly:
the untruncated $s=0.30$ filter multiplies the slowest positive mode
$\lambda_+=4.78\times10^{-3}$ by $e^{-3.5\lambda_+^{0.3}}\approx0.49$
against $e^{-3.5\lambda_+}\approx0.98$ for $s=1$.
Remark~\ref{rem:unequal_diffusion_times} shows how changing the time moves
the spectral crossing, so we also evaluate $t_s=t_1\lambda_+^{1-s}$, which
matches the attenuation of the slowest positive mode; the smallest positive
eigenvalue is used because $\lambda_2=0$ on a disconnected graph, and any
positive eigenvalue defines a matching policy. With $t_1=3.5$ this gives
$t_s=0.083$ at $s=0.30$ and $t_s=13.31$ at $s=1.25$;
Table~\ref{tab:graphheat_m2} reports this policy under the same storage
budget.

At one label per class, validation-based exponent selection gives a paired
gain of $2.42\pm1.18$ points and selects $s=1.25$ on $17$ of $30$ splits,
and the fixed orders below one have lower mean accuracy than the classical
control under both time policies. These results characterize the two
specified policies. Selecting the diffusion time on validation data
separately for each order, with the same search for the classical kernel,
would compare the broader filter families; that protocol is used in the
pipeline and Two-Moon experiments (Tables~\ref{tab:cora_main},
\ref{tab:datasets} and~\ref{tab:fractional_order_sensitivity};
Section~\ref{subsec:twomoon} and Table~\ref{tab:fractional_schemes}) but
not within GraphHeat.

\begin{table}[t]
\centering
\small
\caption{Matched per-split difference against the classical exponent on
Cora when each order receives the diffusion time $t_s=t_1\lambda_+^{1-s}$,
with $t_1=3.5$ and $\lambda_+=4.78\times10^{-3}$ the smallest positive
eigenvalue of the Laplacian used in this experiment, $30$ splits. The last
row is the tuned arm, which selects the exponent on validation data.}
\label{tab:graphheat_m2}
\begin{tabular}{lcc}
\toprule
 & $k{=}1$ & $k{=}4$ \\
\midrule
$s=0.30$ & $-18.54 \pm 3.14$ & $-29.72 \pm 2.19$ \\
$s=0.50$ & $-13.81 \pm 2.86$ & $-22.59 \pm 2.00$ \\
$s=0.75$ & $-6.42 \pm 2.02$  & $-7.92 \pm 0.99$ \\
$s=1.25$ & $+0.73 \pm 2.14$  & $-1.08 \pm 0.91$ \\
\midrule
Tuned on validation & $+2.42 \pm 1.18$ & $+0.16 \pm 0.41$ \\
\bottomrule
\end{tabular}
\end{table}

\paragraph{Label-spectral diagnostic.}
Whether weaker attenuation above $\lambda=1$ helps depends on how much
class-discriminating energy lies there. Table~\ref{tab:label_spectral}
reports the diagnostic of Section~\ref{subsec:label_spectral_alignment} for
the complete one-hot label matrix. The null-space share, $0.22$ on Cora and
$0.38$ on CiteSeer, is unaffected by $s$. Under the self-loop-free
convention of this experiment the region above the fixed point carries
$0.118$ of Cora's label energy against $0.646$ below it, and under the $A+I$
convention of the main pipeline $0.068$ against $0.705$; the fixed-order
losses and the selection counts above agree with this imbalance. The
diagnostic uses the complete labels, which are unavailable at training time,
so a rule for choosing the order would need a quantity computable from the
training labels alone; and the unweighted shares do not fix the sign of the
accuracy difference, since even for source-free diffusion of the complete
labels
\[
\frac{\partial}{\partial s}\bigl\|e^{-TL^{s}}\mathsf Y\bigr\|_F^{2}
=
-2T\sum_{\lambda_k>0}
\lambda_k^{s}\log(\lambda_k)\,e^{-2T\lambda_k^{s}}E_k ,
\]
so eigenvalue-dependent weights enter, and the classifier depends in
addition on the masked initialization, the source, the features and the
learned maps. For PubMed under convention $A$ the reported transition-band energy near
$\lambda=1$ is approximately $0.18$. The diagnostic does not order the
datasets by their gain over GAT: the share above the fixed point is largest
on PubMed, where the pipeline has no consistent gain.

\begin{table}[t]
\centering
\small
\caption{Complete one-hot label energy of $L_{\mathrm{sym}}$ relative to the
fixed point $\lambda=1$: $\rho_0$ is the null-space share, unaffected by
$s$, and $\rho_{>1}$ the share that an order below one retains more
strongly than the classical kernel. Both self-loop conventions are shown
since the convention moves $\lambda_{\max}$. The column $\dim\ker$ equals
the number of components only when every degree is positive: CiteSeer's
$48$ isolated vertices contribute the eigenvalue one, not zero, without
self-loops. Cora and CiteSeer use the exact eigendecomposition; the PubMed values are
approximate polynomial estimates, and its $\rho_{=1}$ entries are the
transition-band energy around one for the reported half-width $10^{-3}$.}
\label{tab:label_spectral}
\begin{tabular}{llcccccc}
\toprule
Dataset & Convention & $\dim\ker$  & $\lambda_{\max}$ &
$\rho_0$ & $\rho_{(0,1)}$ & $\rho_{=1}$ & $\rho_{>1}$ \\
\midrule
Cora     & $A$   & $78$  & $2.000$ & $0.216$ & $0.646$ & $0.021$ & $0.118$ \\
Cora     & $A+I$ & $78$  & $1.483$ & $0.224$ & $0.705$ & $0.003$ & $0.068$ \\
CiteSeer & $A$   & $390$ & $2.000$ & $0.383$ & $0.408$ & $0.056$ & $0.153$ \\
CiteSeer & $A+I$ & $438$ & $1.502$ & $0.407$ & $0.490$ & $0.029$ & $0.074$ \\
PubMed   & $A$   & $1$   & $1.986$ & $0.254$ & $0.408$ & $0.181$ & $0.157$ \\
PubMed   & $A+I$ & $1$   & $1.654$ & $0.287$ & $0.619$ & $0.008$ & $0.086$ \\
\bottomrule
\end{tabular}
\end{table}

\paragraph{Cost.}
At the matched storage budget, sparse propagation has the same asymptotic
arithmetic cost for every exponent. With a precomputed eigendecomposition,
each new order and time requires a scalar spectral reweighting and the
assembly of the propagation matrix, followed by sparsification. Polynomial
methods instead incur an approximation cost that depends on the order, the
time and the required tolerance (Table~\ref{tab:cheb} and
Section~\ref{sec:fraccomp}).
\subsection{Sensitivity to the Fractional Order}
\label{subsec:fractional_order_sensitivity}

Table~\ref{tab:fractional_order_sensitivity} varies the order on the
pipeline of Table~\ref{tab:cora_main}: label-space diffusion with the
compatible source on the GAT-refined graph, Cora, $4$ labels per class, $10$
splits. For each $s\in\{0.30,0.50,0.75,1.00,1.25\}$ it reports the test
accuracy of the single pass (Scheme~2) and of the self-training iteration
(Algorithm~\ref{alg:heat_kernel_self_training} with the threshold schedule
of Section~\ref{subsec:implementation_details}), the matched per-split
differences against $s=1$ and against the single pass, the number of splits
with at least one promotion and the mean number of promoted nodes, the
over-smoothing index of Definition~\ref{def:osi} applied to the homogeneous
part of the state at the selected time,
\[
\mathrm{OSI}_h(s,T)
=
\frac{\|(I-\Pi_0)\,e^{-TL^s}Y_0\|_F}{\|(I-\Pi_0)Y_0\|_F}
\le e^{-T\lambda_+^s},
\]
which is the quantity bounded by Proposition~\ref{prop:osi}, and the share
of splits in which validation selected the long-time limit $t=\infty$.

\begin{table}[t]
\centering
\footnotesize
\setlength{\tabcolsep}{3pt}
\caption{Sensitivity to the order $s$ on Cora at $4$ labels per class, $10$
splits, label-space diffusion on the GAT-refined graph (the pipeline of
Table~\ref{tab:cora_main}). Single pass is the fixed compatible source;
self-training is Algorithm~\ref{alg:heat_kernel_self_training} with the
threshold schedule of Section~\ref{subsec:implementation_details}. $\Delta$
are
matched per-split differences with $95\%$ $t$-intervals. $\mathrm{OSI}_h$
is the over-smoothing index of the homogeneous part $e^{-tL^s}Y_0$ at the
selected $t$, bounded by $e^{-t\lambda_+^s}$ (Proposition~\ref{prop:osi});
the last column is the share of splits in which validation selected
$t=\infty$.}
\label{tab:fractional_order_sensitivity}
\resizebox{\textwidth}{!}{%
\begin{tabular}{cccccccc}
\toprule
$s$
& \begin{tabular}[c]{@{}c@{}}Single\\pass\end{tabular}
& \begin{tabular}[c]{@{}c@{}}$\Delta$ vs.\\$s{=}1$\end{tabular}
& \begin{tabular}[c]{@{}c@{}}Self-\\training\end{tabular}
& \begin{tabular}[c]{@{}c@{}}$\Delta$ vs.\\single pass\end{tabular}
& \begin{tabular}[c]{@{}c@{}}Splits with\\promotion\end{tabular}
& \begin{tabular}[c]{@{}c@{}}Promoted\\nodes\end{tabular}
& \begin{tabular}[c]{@{}c@{}}$\mathrm{OSI}_h$ /\\share $t{=}\infty$\end{tabular} \\
\midrule
$0.30$ & $71.1 \pm 3.2$ & $-0.9 \pm 0.5$ & $71.1 \pm 3.2$ & $0$            & $0/10$ & $0$ & $0.01$ / $0.00$ \\
$0.50$ & $71.5 \pm 3.1$ & $-0.5 \pm 0.4$ & $71.5 \pm 3.1$ & $0$            & $0/10$ & $0$ & $0.02$ / $0.20$ \\
$0.75$ & $71.8 \pm 3.0$ & $-0.2 \pm 0.2$ & $71.8 \pm 3.0$ & $0$            & $0/10$ & $0$ & $0.02$ / $0.50$ \\
$1.00$ & $72.0 \pm 2.9$ & $0$            & $72.0 \pm 2.9$ & $0$            & $0/10$ & $0$ & $0.02$ / $0.60$ \\
$1.25$ & $72.0 \pm 2.8$ & $-0.0 \pm 0.4$ & $72.1 \pm 2.4$ & $+0.1 \pm 0.6$ & $2/10$ & $9$ & $0.02$ / $0.80$ \\
\bottomrule
\end{tabular}}%
\end{table}

The source-driven pipeline maintains mean accuracy between $71.1\%$ and
$72.0\%$ across the tested orders, with the time selected separately for
each order. The paired differences from $s=1$ are $-0.9\pm0.5$
percentage points at $s=0.30$ and $-0.5\pm0.4$ at $s=0.50$. The selected
horizon changes more markedly: the share of splits selecting $t=\infty$
rises from zero at $s=0.30$ to $80\%$ at $s=1.25$. At the selected times,
$\mathrm{OSI}_h\le0.02$ describes strong attenuation of the initial
non-null component. Sustained supervision enters through the source
response, whose non-null long-time limit is $(L^s)^\dagger F$
(Corollary~\ref{cor:source_nonnull}). These results illustrate
classification with long propagation horizons, while the order sets the
relative equilibrium weights $\lambda_k^{-s}$.

On Cora self-training leaves the fixed-source result unchanged at the
tested orders $s\le1$: no node is promoted, and since every eligible node
whose score crosses the threshold is promoted, no eligible node's score
crosses it at the evaluated rounds. At $s=1.25$
promotion occurs in two of ten splits, with nine promoted nodes on average
(Table~\ref{tab:fractional_order_sensitivity}), and the paired accuracy
change is $+0.1\pm0.6$ points. The Two-Moon results likewise show small,
configuration-dependent changes (Section~\ref{subsec:twomoon}).

\subsection{Heterophily and the Sign of the Fractional Effect}
\label{subsec:heterophily}

The heterophily experiments examine fractional propagation when the class
labels align weakly with the local graph neighbourhoods. At a common time
smaller orders retain more of the modes above $\lambda=1$
(Proposition~\ref{prop:rotation}), and both the heat multiplier and the
equilibrium source multiplier remain decreasing functions of frequency.
Table~\ref{tab:hetero} compares label-space diffusion on the original graph
with GAT. The displayed order maximizes mean test accuracy over the tested
grid, so the table is an exploratory best-observed summary subject to
selection bias.

\begin{table}[t]
\centering
\small
\caption{Heterophily benchmarks at $k=5$ labels per class with edge
homophily $H$; label-space diffusion on the original graph. Accuracy is the
mean $\pm$ normal-approximation $95\%$ interval over $10$ splits with
$\min(200,n/4)$ validation and $n/3$ test nodes (Texas $45/61$, Wisconsin
$62/83$, Chameleon $200/759$, Squirrel $200/1733$); the GAT is trained for
$150$ epochs. $s^{\star}$ maximizes mean test accuracy over the tested grid; the
displayed intervals are unadjusted for this selection.}
\label{tab:hetero}
\begin{tabular}{lcccc}
\toprule
Dataset & Homophily $H$ & GAT & Fractional ($s^{\star}$) & $s^{\star}$ \\
\midrule
Chameleon & $0.23$ & $37.3 \pm 1.5$ & $33.1 \pm 1.7$ & $0.50$ \\
Squirrel  & $0.22$ & $24.8 \pm 1.0$ & $23.2 \pm 0.8$ & $1.00$ \\
Texas     & $0.06$ & $48.4 \pm 4.5$ & $32.5 \pm 6.4$ & $0.30$ \\
Wisconsin & $0.18$ & $42.2 \pm 4.2$ & $33.5 \pm 4.0$ & $0.30$ \\
\bottomrule
\end{tabular}
\end{table}

GAT has the higher mean accuracy on all four graphs, by $1.6$ points on
Squirrel, $4.2$ on Chameleon, $8.7$ on Wisconsin and $15.9$ on Texas. The
best-observed order is below one on three graphs and equals one on
Squirrel, so the tested graphs show no uniform preference for an order
below one under heterophily, consistent with the filter remaining low-pass
for every order. The small test sets on Texas and Wisconsin ($61$ and $83$
nodes per split) limit the precision of these comparisons.

\subsection{Robustness to Graph Noise}
\label{subsec:graph_noise}

Table~\ref{tab:rewiring_full_ablation} rewires a fixed fraction of the Cora
edges at random, preserving the total number of edges, at $4$ labels per
class over $10$ random splits; all non-GAT hyperparameters are selected on
validation nodes.

\begin{table}[t]
\centering
\caption{Robustness to random edge rewiring on Cora at \(4\) labels per
class; mean \(\pm\) normal-approximation \(95\%\) interval over \(10\)
random splits. ``Refined'' is the GAT-refined graph and ``Label-FHK,
GAT-refined graph'' the pipeline of Table~\ref{tab:cora_main};
``residual'' applies the decoder to $Z_T+Z_0$ and $\kappa$ is the
validation-selected source strength (Section~\ref{subsec:neural_fhkpp}).
These runs use an earlier version of the protocol, before the componentwise
source, the $t=\infty$ grid point and the GAT training length of
Section~\ref{subsec:implementation_details}, so the GAT baseline differs
slightly from Table~\ref{tab:cora_main}.}
\label{tab:rewiring_full_ablation}
\resizebox{\textwidth}{!}{%
\begin{tabular}{lccccc}
\toprule
Method & 0\% & 5\% & 10\% & 20\% & 40\% \\
\midrule
GAT baseline & 70.6 $\pm$ 2.7 & 66.9 $\pm$ 3.1 & 66.9 $\pm$ 2.0 & 58.1 $\pm$ 2.9 & 48.3 $\pm$ 4.5 \\
Label-FHK, original graph & 60.0 $\pm$ 3.2 & 53.0 $\pm$ 3.2 & 46.9 $\pm$ 2.7 & 37.5 $\pm$ 3.4 & 25.5 $\pm$ 3.0 \\
Label-FHK, GAT-refined graph & 68.3 $\pm$ 2.8 & 65.4 $\pm$ 2.7 & 64.1 $\pm$ 2.5 & 57.2 $\pm$ 2.5 & 47.2 $\pm$ 3.2 \\
Neural FHK++, original graph & 70.1 $\pm$ 1.7 & 66.1 $\pm$ 3.5 & 65.1 $\pm$ 2.3 & 58.1 $\pm$ 2.4 & 47.9 $\pm$ 3.5 \\
Neural FHK++, GAT-refined graph & 69.0 $\pm$ 2.2 & 67.5 $\pm$ 2.2 & 65.3 $\pm$ 1.9 & 58.1 $\pm$ 2.8 & 49.1 $\pm$ 3.3 \\
Neural FHK++, refined + residual & 68.0 $\pm$ 2.5 & 64.9 $\pm$ 2.8 & 64.5 $\pm$ 2.0 & 56.2 $\pm$ 3.2 & 47.0 $\pm$ 3.8 \\
Neural FHK++, refined + residual + \(\kappa\) & 68.3 $\pm$ 1.9 & 64.7 $\pm$ 3.0 & 64.4 $\pm$ 2.5 & 57.3 $\pm$ 2.6 & 47.4 $\pm$ 4.1 \\
\bottomrule
\end{tabular}
}%
\end{table}

These measurements use the earlier implementation identified in the
caption, whose source was not centered componentwise, so that on the
disconnected Cora graph $\Pi_0F\neq0$ in general and the null-space
component of the state grew linearly in $t$ (Theorem~\ref{thm:mass}), and
whose time-selection and training settings differ from
Section~\ref{subsec:implementation_details}. Under that protocol, GAT-based
graph refinement improves label-space diffusion at every rewiring level, by
$8.3$ points at zero rewiring and $21.7$ points at $40\%$ rewiring, and
Neural FHK++ has mean accuracies close to GAT at every noise level, on the
original graph as well as on the refined one. Neither the residual
connection nor the tuned source strength $\kappa$ improves the mean. These
observations characterize the role of learned representations and graph
refinement in the earlier implementation.

\subsection{Comparison with Literature Baselines}
\label{subsec:literature_baselines}

Under the protocol of this paper the controlled comparisons are with GAT and
GraphHeat (Tables~\ref{tab:cora_main}, \ref{tab:datasets}
and~\ref{tab:rewiring_full_ablation}).
GRAND~\cite{chamberlain2021grand} and GRAND++~\cite{thorpe2022grandpp} are
the closest methods in formulation, since GRAND++ adds a source term to the
diffusion equation; their published low-label results are reproduced as
reference values in Table~\ref{tab:grandpp_literature_low_label}.
Table~\ref{tab:literature_comparison} lists representative Cora results
under the standard $20$-labels-per-class setting ($140$ labeled nodes), from
classical graph-based methods to the diffusion-based
APPNP~\cite{klicpera2019predict} and the deep GCNII~\cite{chen2020simple}.
Splits, preprocessing, validation protocols and tuning differ across papers,
so these values are context rather than head-to-head comparisons. The last
block of Table~\ref{tab:literature_comparison} is Neural FHK++ retrained
under the Cora protocol of Section~\ref{subsec:implementation_details} at
$20$ labels per class, using the variant with the highest validation
accuracy in Table~\ref{tab:rewiring_full_ablation}, which was obtained at
$4$ labels per class.

\begin{table}[htbp]
\centering
\small
\caption{Representative published results on Cora under the standard
\(20\)-labels-per-class setting, together with the controlled result of the
best Neural FHK++ source-driven variant. Literature values are contextual
reference points, not controlled comparisons. The values for Label
Propagation, DeepWalk and Manifold Regularization are those reported
in~\cite{KipfWelling2017GCN}, taken in turn from~\cite{yang2016planetoid}.}
\label{tab:literature_comparison}
\begin{tabular}{lccc}
\toprule
\textbf{Method} & \textbf{Type} & \textbf{Total labels} & \textbf{Reported test accuracy} \\
\midrule
\multicolumn{4}{c}{\textit{Classical graph-based methods}} \\
\midrule
Label Propagation~\cite{zhu2003semi} & Classical & \(140\) & \(68.0\%\) \\
DeepWalk~\cite{perozzi2014deepwalk} & Classical & \(140\) & \(67.2\%\) \\
Manifold Regularization~\cite{belkin2006manifold} & Classical & \(140\) & \(59.5\%\) \\
\midrule
\multicolumn{4}{c}{\textit{Graph neural networks}} \\
\midrule
GCN~\cite{KipfWelling2017GCN} & GNN & \(140\) & \(81.5\%\) \\
GAT~\cite{velickovic2018graph} & GNN & \(140\) & \(83.0\%\) \\
SGC~\cite{wu2019sgc} & GNN & \(140\) & \(81.0\%\) \\
APPNP~\cite{klicpera2019predict} & Diffusion GNN & \(140\) & \(83.3\%\) \\
DropEdge~\cite{rong2020dropedge} & Regularized GNN & \(140\) & \(82.8\%\) \\
\midrule
\multicolumn{4}{c}{\textit{Recent advanced methods}} \\
\midrule
GCNII~\cite{chen2020simple} & Deep GNN & \(140\) & \(85.5\%\) \\
\midrule
\multicolumn{4}{c}{\textit{Controlled result from this paper}} \\
\midrule
Ours, Neural FHK++ with source & Fractional diffusion & \(140\) & \(\mathrm{81.7}\pm\mathrm{1.5}\%\) \\
\bottomrule
\end{tabular}
\end{table}

\begin{table}[t]
\centering
\small
\caption{Published GRAND++ low-label results of Thorpe et
al.~\cite{thorpe2022grandpp}, in percent, included as literature references
only; they were not obtained under the splits, preprocessing or validation
protocol of this paper.}
\label{tab:grandpp_literature_low_label}
\begin{tabular}{lccccc}
\toprule
Dataset & \(1\) label/class & \(2\) labels/class & \(5\) labels/class & \(10\) labels/class & \(20\) labels/class \\
\midrule
Cora
& \(54.94 \pm 16.09\)
& \(66.92 \pm 10.04\)
& \(77.80 \pm 4.46\)
& \(80.86 \pm 2.99\)
& \(82.95 \pm 1.37\) \\
CiteSeer
& \(58.95 \pm 9.59\)
& \(64.98 \pm 8.31\)
& \(70.03 \pm 3.63\)
& \(72.34 \pm 2.42\)
& \(73.53 \pm 3.31\) \\
PubMed
& \(65.94 \pm 4.87\)
& \(69.31 \pm 4.87\)
& \(71.99 \pm 1.91\)
& \(75.13 \pm 3.88\)
& \(79.16 \pm 1.37\) \\
\bottomrule
\end{tabular}
\end{table}

At twenty labels per class the GraphHeat reproduction reaches $83.58$ on
the fixed public split, while the proposed pipelines reach $79.1$ (label
space, Table~\ref{tab:cora_main}) and $81.7$ (Neural FHK++,
Table~\ref{tab:literature_comparison}) on random splits, where the GAT
baseline is also lower ($80.2$ against the published $83.0$). Under the
matched protocols of this paper the largest mean gains over GAT occur at
the sparsest training budgets (Tables~\ref{tab:cora_main}
and~\ref{tab:datasets}).
 \section{Limitations and Failure Modes}
\label{sec:limitations}

\emph{Order and time selection.} Useful spectral weighting depends on the
graph, the class signal and the source. At a common time, reducing $s$
below one increases damping on $0<\lambda<1$ and decreases it on
$\lambda>1$ (Proposition~\ref{prop:rotation}); independently selected
times shift the crossing (Remark~\ref{rem:unequal_diffusion_times}).
These relations guide order-and-time selection, while a fixed compatible
nonzero source sustains the non-null equilibrium for every $s>0$. The
label-space pipeline selects the time per order. The GraphHeat experiments
evaluate two specified time policies with a fixed storage budget; joint
validation selection of order and time and an unsparsified comparison
remain future work.

\emph{Computational cost.} For polynomial approximation on an interval
containing zero, non-integer orders give algebraic Chebyshev convergence
through the endpoint singularity of $e^{-tx^{s}}$
(Section~\ref{sec:fraccomp}), and the required degree depends on the order,
the time and the error tolerance. Subordination and rational Krylov methods
are alternatives for large graphs.

\emph{Self-training.} The optional extension produces small,
configuration-dependent changes and no established benefit in the reported
experiments. Its heuristic confidence score depends on the scale of the
propagated scores, so calibration and threshold selection are directions
for further study, and incorrect pseudo-labels can be reinforced in
subsequent rounds~\cite{arazo2020pseudo}.

\emph{Graph quality.} Propagation is most useful when the graph geometry
reflects the class signal. The heterophily results show weaker performance
than GAT on the tested graphs, and the rewiring study, obtained with the
earlier implementation specified in the caption of
Table~\ref{tab:rewiring_full_ablation}, shows the value of learned graph
refinement under edge perturbations.
 
 \section{Conclusion}\label{sec:conclusion}
We developed a source-driven fractional heat-kernel framework that combines
nonlocal graph propagation with sustained label supervision. A fixed nonzero
compatible source prevents asymptotic collapse into the Laplacian null
space: the non-null state converges to $(L^s)^\dagger F\ne0$. This gives a
precise mechanism for mitigating oversmoothing at long diffusion times. The
fractional order controls the relative modal decay rates and the
low-frequency preference of the equilibrium response, while the diffusion
time sets the propagation horizon. Componentwise compatibility extends the
conservation and equilibrium analysis to normalized and disconnected
graphs; deflation and approximation results support the construction and
evaluation of the propagators.

The experiments identify the roles of these components under sparse
supervision. On Two-Moon, fractional orders improve source-free diffusion,
and compatible sources produce large gains at the sparsest label budgets,
including high accuracy at the long-time limit. On Cora and CiteSeer, the
source-driven, GAT-refined pipeline improves mean accuracy over GAT by
$9.2$ and $8.0$ percentage points at one training label per class, with
closely comparable classical and fractional kernels. Within GraphHeat,
validation-based exponent selection at a common time gives a paired gain
of $1.45$ percentage points at that budget. The framework thus provides
both a mathematical account of sustained supervision under diffusion and a
practical means of selecting its spectral weighting. Joint order-and-time
selection within GraphHeat, an unsparsified comparison of the exponents and
an evaluation of deflation are natural next steps.

\section*{Acknowledgments}
VK acknowledges support from the Czech National Science Foundation under Project 24-11664S.

\bibliographystyle{unsrt}
\bibliography{Reff}

\end{document}